\documentclass[10pt,letterpaper,twocolumn,twoside]{IEEEtran}

\usepackage[utf8]{inputenc} 
\usepackage[T1]{fontenc}    
\usepackage[hyphens]{url}     
\usepackage{hyperref}       
\usepackage{tabularx}
\usepackage{booktabs}       
\usepackage{amsfonts}       
\usepackage{amssymb}        
\usepackage{nicefrac}       
\usepackage{microtype}      
\usepackage{bm}				
\usepackage{cite}			
\usepackage{graphicx}		
\usepackage{mathtools}		

\usepackage{algorithm}
\usepackage{algcompatible}
\usepackage{nicefrac}
\usepackage{comment}
\usepackage{amsthm}			
\usepackage{enumitem}		
\usepackage{multirow}       
\usepackage{tabularx}	    
\usepackage{hhline}
\usepackage{arydshln}		
\usepackage{xcolor}
\usepackage{mathtools,nccmath}
\definecolor{lightgreen}{rgb}{.9,1,.9}
\definecolor{red}{rgb}{1,0,0}

\newcolumntype{L}[1]{>{\raggedright\arraybackslash}p{#1}}
\newcolumntype{C}[1]{>{\centering\arraybackslash}p{#1}}
\newcolumntype{R}[1]{>{\raggedleft\arraybackslash}p{#1}}

\theoremstyle{plain} 
\newtheorem{proposition}{Proposition}

\newtheorem{lemma}{Lemma}

\newtheorem{assumption}{Assumption}

\def\argmin{\mathop{\mathsf{arg\,min}}} 

\def\lim{\mathop{\mathsf{lim}}} 
 
\def\max{\mathop{\mathsf{max}}} 

\def\inf{\mathop{\mathsf{inf}}}

\def\log{\mathrm{log}}

\newcommand{\bA}{\mathbf{A}}
\newcommand{\by}{\mathbf{y}}
\newcommand{\bY}{\mathbf{Y}}
\newcommand{\bx}{\mathbf{x}}
\newcommand{\bX}{\mathbf{X}}
\newcommand{\bz}{\mathbf{z}}
\newcommand{\bZ}{\mathbf{Z}}
\newcommand{\bv}{\mathbf{v}}
\newcommand{\bV}{\mathbf{V}}

\newcommand{\bS}{\mathbf{S}}
\newcommand{\bQ}{\mathbf{Q}}
\newcommand{\bB}{\mathbf{B}}
\newcommand{\bI}{\mathbf{I}}

\newcommand{\ac}[1]{\left\{{#1}\right\}}
\newcommand{\normn}[1]{\|{#1}\|} 
 
\newcommand{\normBig}[1]{\Big\|{#1}\Big\|} 

\newcommand{\dd}{\mathrm{d}}
\newcommand{\grad}{\nabla}

\newcommand{\br}[1]{\left[{#1}\right]}
\newcommand{\mcb}[1]{\mathcal{B}(#1)}

\newcommand{\pr}[1]{\left({#1}\right)}

\begin{document}

\title{Regularization by denoising: Bayesian model and Langevin-within-split Gibbs sampling}

\author{
Elhadji~C.~Faye, Mame~Diarra~Fall and Nicolas Dobigeon,~\IEEEmembership{Senior Member,~IEEE} 

\thanks{E.~C.~Faye and M.~D.~Fall are with Institut Denis Poisson, University of Orléans, Orléans, France (e-mail:
\{elhadji-cisse.faye, diarra.fall\}@univ-orleans.fr).}
\thanks{N. Dobigeon is with University of Toulouse, IRIT/INP-ENSEEIHT,
CNRS, 2 rue Charles Camichel, BP 7122, 31071 Toulouse Cedex 7, France (e-mail: Nicolas.Dobigeon@enseeiht.fr).}
\thanks{Part of this work was supported by the Artificial Natural Intelligence Toulouse Institute (ANITI, ANR-19-PI3A-0004), the AI.iO Project (ANR-20-THIA-0017) and the BACKUP project (ANR-23-CE40-0018-01).}}

\maketitle

\begin{abstract}
This paper introduces a Bayesian framework for image inversion by deriving a probabilistic counterpart to the regularization-by-denoising (RED) paradigm. It additionally implements a Monte Carlo algorithm specifically tailored for sampling from the resulting posterior distribution, based on an asymptotically exact data augmentation (AXDA). The proposed algorithm is an approximate instance of split Gibbs sampling (SGS) which embeds one Langevin Monte Carlo step. The proposed method is applied to common imaging tasks such as deblurring, inpainting and super-resolution, demonstrating its efficacy through extensive numerical experiments. These contributions advance Bayesian inference in imaging by leveraging data-driven regularization strategies within a probabilistic framework.
\end{abstract}

\section{Introduction}
 This paper is interested in conducting Bayesian inference about an image $\bx \in \mathbb{R}^n$ given the measurements $\by \in \mathbb{R}^m$ related to $\bx$ through a statistical model specified by the likelihood function
\begin{equation}
\label{eq:likelihood}
     p(\by|\bx) \propto \exp\left[-f(\bx,\by)\right].
\end{equation}
In \eqref{eq:likelihood}, the potential function $f(\bx, \by)$ is a fidelity term, i.e., accounting for the consistency of $\bx$ with respect to (w.r.t.) the measured data $\by$. In what follows, this potential function will be assumed to be convex and $L_f$-smooth, i.e., continuously differentiable and its gradient is Lipschitz continuous with Lipschitz constant $L_f$. This problem is in line with the most frequently encountered imaging inverse problems such as denoising, deblurring, and inpainting relying on a  linear forward model and a Gaussian perturbation. For such tasks,  the potential function writes $f(\bx, \by)=\frac{1}{2\sigma^2}\|\bA \bx - \by \|^2_2$ where $\bA \in \mathbb{R}^{m \times n}$ is the degradation matrix. Estimating $\bx$ from $\by$ is generally an ill-posed or, at least, ill-conditioned problem. The Bayesian paradigm consists in assigning a prior distribution to $\bx$, which summarizes the prior knowledge about $\bx$ and acts as a regularization. This prior distribution writes
\begin{equation}
    p(\bx) \propto \exp\left[-\beta g(\bx)\right]
\end{equation}
where $g: \mathbb{R}^n \rightarrow \mathbb{R}$ stands for the regularization term and the parameter $\beta >0$ controls the amount of regularization enforced by the prior distribution. The posterior distribution $p(\bx|\by)$ is derived from the likelihood $p(\by | \bx)$ and the prior distribution $p(\bx)$ using the Bayes' rule
\begin{equation}
    p(\bx|\by) \propto  \exp \left[ -f(\bx,\by) - \beta g(\bx) \right].
    \label{posterior}
\end{equation}
This posterior distribution provides a comprehensive description of the solutions and allows various Bayesian estimators and uncertainty measures to be derived. In particular, computing the maximum a posteriori (MAP) estimator boils down to solving the minimization problem
\begin{equation}
\label{eq:MAP}
    \operatornamewithlimits{min}_{\bx} f(\bx,\by) + \beta g(\bx). 
\end{equation}
Numerous works from the literature have focused on the difficult task of designing a relevant prior distribution $p(\bx)$ or, equivalently, a relevant potential function  $g(\bx)$. These regularizations usually promote specific expected or desired properties about $\bx$. More specifically, conventional optimization methods solving \eqref{eq:MAP} are generally based on explicit model-based regularizations, such as total variation (TV) promoting piecewise constant behavior \cite{rudin1992nonlinear}, Sobolev promoting smooth content \cite{karl2005regularization} or sparsity-promoting regularizations based on the use of $\ell_p$-norm with $p\leq 1$ \cite{cao2016image,afonso2010augmented}. However, designing an appropriate model-based regularization remains an empirical and subjective choice. Moreover, their ability of characterizing complex image structures is generally limited or comes at the price of a significant increase of the resulting algorithmic burden. More recently, a different route has been taken by devising smart strategies avoiding the ad hoc design of explicit model-based regularizations. The seminal work by Venkatakrishnan \emph{et al.} has introduced the concept of \emph{plug-and-play} (PnP) as an implicit prior \cite{venkatakrishnan2013plug}. This framework naturally emerges whenever the algorithmic scheme designed to solve \eqref{eq:MAP} embeds the proximal operator associated with $g(\cdot)$. Possible schemes include the alternating direction method of multipliers (ADMM) \cite{afonso2010augmented}, half quadratic splitting (HQS) \cite{Geman1995nonlinear} or Douglas-Rachford splitting \cite{Douglas1956numerical}. Interestingly, this proximal mapping can be interpreted as a denoising task under the assumption of an additive white Gaussian noise. PnP approaches replace this proximal step by a more general denoiser $\mathsf{D}: \mathbb{R}^n \rightarrow \mathbb{R}^n$, including non-local means (NLM) \cite{buades2005non}, block-matching and 3D filtering (BM3D) \cite{dabov2007image} or any more recently proposed learning-based denoisers such as DnCNN \cite{zhang2017beyond} or DRUNet \cite{zhang2021plug}. Thanks to its effectiveness and its simplicity, this framework has gained in popularity for a wide range of applications in the context of imaging problems \cite{zhang2021plug, chan2016plug, ryu2019plug, hurault2022proximal}. In the same vein as PnP, the regularization-by-denoising (RED) framework defines an explicit image-adaptive Laplacian-based prior which only relies on the ability of performing a  denoising task \cite{romano2017little}. Empirically, RED has shown to outperform PnP-based approaches and has motivated several subsequent research works \cite{reehorst2018regularization,cohen2021regularization, hurault2022proximal}.

All the variational approaches discussed above treat $\bx$ in a deterministic way and generally produce only point estimates approximating the solution of the minimization problem \eqref{eq:MAP}. As an alternative, the Bayesian framework models the image $\bx$ as a random variable and generally seeks a comprehensive description of the posterior distribution $p(\bx|\by)$. As such, Bayesian methods are able to go beyond a sole point estimation, by enhancing it with a quantification of uncertainty in a probabilistic manner in terms of variance and credibility intervals. This ability to quantify uncertainty is particularly useful for decision-making and reliability assessment \cite{bardsley2018computational,cai2018uncertainty}. Exploring the posterior distribution is generally carried out by generating samples asymptotically distributed according to this target distribution using Markov chain Monte Carlo (MCMC) methods. Most of the works dedicated to the development of MCMC algorithms for inverse problems in imaging relies on conventional model-based prior distributions. As their deterministic counterparts, they encode expected characteristics of the image prescribed beforehand and chosen based on quite empirical arguments. Very few recent works have attempted to depart from this paradigm by incorporating data- or task-driven regularizations as prior distributions. For instance, available training samples can be used to learn a mapping from an instrumental latent distribution towards the image prior. Benefiting from advances in the machine learning literature, this mapping can be chosen as a deep generative model, such as a variational autoencoder \cite{holden2022bayesian} or a normalizing flow \cite{cai2023nf}. Devising a PnP prior in the context of Monte Carlo sampling has been investigated in  \cite{laumont2022bayesian}, resulting in the so-called PnP unadjusted Langevin algorithm (PnP-ULA). Its rationale follows the same motivation as its deterministic counterpart, namely avoiding the explicit definition of the prior distribution by the ability of performing a denoising task. Surprisingly, although that the RED approach has shown to outperform PnP when embedded into a variational framework, up to authors' knowledge, there is no equivalent for the RED paradigm, i.e., RED has never been formulated into a Bayesian framework and embedded into a Monte Carlo algorithm.

The objective of this paper is to fill this gap. More precisely, the main contributions reported hereafter can be summarized as follows. First, Section \ref{sec:RED} introduces a probabilistic counterpart of RED by defining a new distribution that can be subsequently chosen as a prior distribution in a Bayesian inversion task. Then, Section \ref{sec:proposed_method} introduces a new Monte Carlo algorithm that is shown to be particularly well suited to sample from the resulting posterior distribution. It follows an asymptotically exact data augmentation (AXDA) scheme \cite{vono2020asymptotically}, resulting in a nonstandard instance of the split Gibbs sampler (SGS) \cite{vono2019split}. This sampling scheme is thus accompanied by a thorough theoretical analysis to ensure and quantify its convergence. The rationale of the proposed approach is also put into perspective w.r.t.~recently proposed Monte Carlo algorithms, in particular PnP-ULA, drawing some connections between AXDA and RED leveraging the Tweedie's formula \cite{efron2011tweedie}. Then, the proposed algorithm is instantiated to solve three ubiquitous inversion tasks, namely deblurring, inpainting and super-resolution. Extensive numerical experiments are conducted in Section \ref{sec:experiments} to compare the performance of the proposed algorithm to state-of-the-art variational and Monte Carlo methods.\\

\noindent\textbf{Notations and conventions.}
The Euclidean norm on $\mathbb{R}^n$ is denoted by $\|\cdot\|$.
We denote by $\mathcal{N}(\boldsymbol{\mu},\bQ^{-1})$ the Gaussian distribution with mean vector $\boldsymbol{\mu}$ and precision matrix $\bQ$. The $(n \times n)$-identity matrix is denoted ${\bI}_n$. For any matrix $\bS \in \mathcal{M}_n\left(\mathbb{R}\right)$, if we denote  $\boldsymbol{0}$ the zero matrix, the notation $\boldsymbol{0} \preccurlyeq \bS$ means that $\bS$ is semi-definitive positive. The Wasserstein distance of order $2$ between two probability measures $\tau$ and $\tau'$ on $\mathbb{R}^n$ with finite $2$-moments is defined by $W_2 (\tau, \tau')~=~(\inf_{\zeta \in \mathcal{T}(\tau,\tau')}~\int_{\mathbb{R}^n \times \mathbb{R}^n}~\|~V~-~ V'~\|^2~\dd\zeta(V,V'))^{\nicefrac{1}{2}}$, where $\mathcal{T}(\tau, \tau')$ is the set of transport plans of $\tau$ and $\tau'$.

\section{Bayesian formulation of RED inversion}\label{sec:RED}
This section starts by recalling some background about RED. Then it proposes a probabilistic counterpart of the regularization, that can be subsequently used as a prior within a Bayesian framework.

\subsection{Regularization by denoising (RED)}\label{subsec:RED}
The RED engine defines $g(\cdot)$ as the explicit image-adaptive Laplacian-based  potential \cite{romano2017little}
\begin{equation}
    g_{\textrm{red}}(\bx)= \dfrac{1}{2}\bx^\top  \left( \bx - \mathsf{D}_{\nu}(\bx) \right)
    \label{RED}
\end{equation}
where $\mathsf{D}_{\nu}: \mathbb{R}^n \rightarrow \mathbb{R}^n$ is a denoiser with $\nu$ controlling the denoising strength, designed for the removal of additive white Gaussian noise. Although it offers a significant flexibility in the choice of the denoisers that can be used, RED requires $\mathsf{D}_{\nu}(\cdot)$ to obey the following assumptions, referred to as the RED conditions.
\begin{enumerate}[label=(\texttt{C\arabic*})]
\itemindent=5pt
    \item \label{RedCond:C1} \emph{Local homogeneity}: $\forall \bx \in \mathbb{R}^n$, 
      \begin{equation}
          \mathsf{D}_{\nu}\left( (1 + \epsilon)\bx \right) = (1 + \epsilon) \mathsf{D}_{\nu}(\bx) \label{eq:LH}
      \end{equation}
      for any sufficiently small $\epsilon > 0$. 
    \item \label{RedCond:C2} \emph{Differentiability}: the denoiser $\mathsf{D}_{\nu}(\cdot)$ is differentiable with Jacobian denoted $\nabla \mathsf{D}_{\nu}(\cdot)$.
    \item \label{RedCond:C3} \emph{Jacobian symmetry} \cite{reehorst2018regularization}: $\forall \bx \in \mathbb{R}^n$, $\nabla \mathsf{D}_{\nu}(\bx)^\top = \nabla \mathsf{D}_{\nu}(\bx)$.
    \item \label{RedCond:C4} \emph{Strong passivity}: the Jacobian spectral radius satisfies $\eta \left( \nabla \mathsf{D}_{\nu}(\bx) \right) \le 1$.
\end{enumerate}
The major implication of local homogeneity \ref{RedCond:C1} is that the directional derivative of $\mathsf{D}_{\nu}(\cdot)$ along $\bx$ can be computed by applying the denoiser itself, i.e.,
\begin{equation}\label{eq:homogeneity}
     \nabla \mathsf{D}_{\nu}(\bx)\bx = \mathsf{D}_{\nu}(\bx).
\end{equation}
The Jacobian symmetry \ref{RedCond:C3} and the strong passivity \ref{RedCond:C4} ensure that applying the denoiser does not increase the norm of the input:
\begin{equation}
    \| \mathsf{D}_{\nu}(\bx)\| = \| \nabla \mathsf{D}_{\nu}(\bx)\bx\| \le \eta \left( \nabla \mathsf{D}_{\nu}(\bx) \right) \cdot \| \bx \| \le \| \bx \|.
\end{equation}
Interestingly, two additional keys and highly beneficial properties follow: \emph{i)} the RED potential $g_{\textrm{red}}(\cdot)$ is a convex functional and \emph{ii)} the gradient of $g_{\textrm{red}}(\cdot)$ is expressed as the denoising residual  
\begin{equation}
    \nabla g_{\textrm{red}}(\bx)  = R(\bx) = \bx - \mathsf{D}_{\nu}(\bx)
    \label{grad_g}
\end{equation}
which avoids differentiating the denoising operation itself. Thus, one of the most appealing opportunity offered by RED is its ability to embed powerful denoisers, such as those based on deep neural networks, without requiring to differentiate them. It is worth noting that if the denoising function $\mathsf{D}_{\nu}(\cdot)$ does not meet the condition \ref{RedCond:C3}, then there is no regularizer $g(\cdot)$ whose gradient can be written as the residual $R(\cdot)$ \cite{reehorst2018regularization}. Unfortunately many popular denoisers, such as trainable nonlinear reaction-diffusion (TNRD), NLM, BM3D, and DnCNN, are characterized by non-symmetric Jacobian. Yet, RED-based restoration algorithms are shown to empirically converge and to reach excellent performance when solving various inverse problems even when those conditions are partially satisfied \cite{romano2017little, cohen2021regularization}.

\subsection{Probabilistic counterpart of RED}
To formulate the RED-based inversion within a statistical framework, one requirement consists in introducing a prior distribution defined from the RED potential $g_{\textrm{red}}(\cdot)$ given by \eqref{RED}. More precisely, one defines 
\begin{eqnarray}
    p_{\textrm{red}}(\bx) &\propto&  \exp \left[ -\dfrac{\beta}{2}\bx^\top  \left( \bx - \mathsf{D}_{\nu}(\bx) \right) \right].
    \label{prior}
\end{eqnarray}
The functional $p_{\textrm{red}}(\cdot)$ does not necessarily define a probability density function (pdf). For $p_{\textrm{red}}(\cdot)$ to be a valid pdf, i.e. $\int_{\mathbb{R}^n} p_{\textrm{red}}(\bx) \dd \bx < \infty$, certain conditions must be satisfied.

\begin{assumption}
\label{ass:well_defined_density}
The matrix $\boldsymbol{\Lambda}(\bx) = \mathbf{I}_n - \nabla \mathsf{D}_{\nu}(\bx)$, $\forall \bx \in \mathbb{R}^n$,  has at least one non-zero eigenvalue.
\end{assumption}

This technical assumption is not restrictive and it is easy to show that it would  be violated only in trivial cases. Indeed, let  $\bar{\bx} \in \mathbb{R}^n$ denote an image such that all the eigenvalues of $\boldsymbol{\Lambda}(\bar{\bx})$ are zero. Then the matrix $\nabla \mathsf{D}_{\nu}(\bar{\bx})$ is symmetric (see RED condition \ref{RedCond:C3}), with real coefficients and all eigenvalues equal to $1$. The spectral theorem  yields $\nabla \mathsf{D}_{\nu}(\bar{\bx})=\bI_n$. From the local homogeneity \ref{RedCond:C1} and its corollary \eqref{eq:homogeneity}, one has $\mathsf{D}_{\nu}(\bar{\bx})=\bar{\bx}$, i.e., $\bar{\bx}$ is already a noise-free image which does not need to be further denoised.

The next result states that mild assumptions are sufficient to guarantee that the function \eqref{prior} defines a proper distribution.
\begin{proposition}
    If Assumption~\ref{ass:well_defined_density} and Conditions \ref{RedCond:C3}--\ref{RedCond:C4} hold, then
    \begin{equation}
         \int_{\mathbb{R}^n} p_{\text{red}}(\bx) \dd \bx < + \infty
    \end{equation}
    and $p_{\textrm{red}}(\cdot)$ in \eqref{prior} defines a proper pdf.
    \label{prop1}
\end{proposition}
\begin{proof}
    See Appendix \ref{proof:prop1}.
\end{proof}

\subsection{RED posterior distribution}
Combining the RED prior $p_{\textrm{red}}(\bx)$ defined by \eqref{prior} and the likelihood function $p(\by|\bx)$ defined by \eqref{eq:likelihood}, the RED posterior distribution $\pi$ of interest writes
\begin{equation}
    \begin{array}{rcl}
       \pi(\bx) &\triangleq& p(\bx|\by) \\
       &\propto& \exp \left[ - f(\bx,\by) -\dfrac{\beta}{2}\bx^\top  \left( \bx - \mathsf{D}_{\nu}(\bx) \right) \right]. 
    \end{array}
    \label{eq:target}
\end{equation}
As stated earlier, deriving the MAP estimator associated with the RED posterior \eqref{eq:target} consists in solving the optimization problem \eqref{eq:MAP}. In the seminal paper \cite{romano2017little}, this problem is tackled thanks to first-order optimization methods such as steepest decent (SD), fixed-point (FP) iteration and ADMM. More recently, it has been reformulated as a convex optimization problem using a projection onto the fixed point set of demicontractive denoisers \cite{cohen2021regularization}. Instead, the work conducted in this manuscript proposes to follow a different route by proposing to sample from this posterior distribution. While these samples offer a comprehensive description of the RED posterior, they can  be subsequently used to derive Bayesian estimators or credibility intervals. Because of the non-standard form of the RED posterior, sampling according to \eqref{eq:target} requires to develop a dedicated algorithm introduced in the following section.

\section{Proposed algorithm}
\label{sec:proposed_method}

\subsection{Langevin-within-split Gibbs sampler}\label{subsec:LwSGS}
Generating samples efficiently from the posterior distribution with pdf $\pi(\bx)$ defined by \eqref{eq:target} is not straightforward, in particular due to the use of the denoiser $\mathsf{D}_{\nu}(\cdot)$. When $\pi(\cdot)$ is proper and smooth with $\bx \mapsto \nabla \log\, \pi(\bx)$ Lipschitz continuous, one solution would consist in resorting to the ULA \cite{durmus2019high}. This strategy will be shown to be intimately related to PnP-ULA in Section \ref{subsec:discussion_PnP}. However, it may suffer from several shortcomings, such as poor mixing properties and higher resulting computational times (see experimental results in Section \ref{subsec:results}). Conversely, the work in this manuscript derives a dedicated Monte Carlo algorithm to sample from a posterior distribution written as \eqref{posterior}. This algorithm will be shown to be particularly well suited to sample from the RED posterior \eqref{eq:target}, i.e., when $g(\cdot) = g_{\textrm{red}}(\cdot)$.

The proposed sampling scheme first leverages an asymptotically exact data augmentation (AXDA) as introduced in \cite{vono2020asymptotically}. Inspired by optimization-flavored counterparts, AXDA employs a variable splitting technique to simplify and speed up the sampling according to possibly complex distributions. More precisely, it introduces an auxiliary variable $\bz \in \mathbb{R}^n$ and considers the augmented distribution 
    \begin{eqnarray}
       \pi_{\rho}(\bx, \bz) &=& p({\bx}, {\bz}|{\by}; \rho^2) \label{split_dist} \\
       &\propto& \exp \left[ - f(\bx,\by) - \beta g(\bz) - \frac{1}{2\rho^2}|| \bx - \bz||^2 \right] \nonumber
    \end{eqnarray}
where $\rho$ is a positive parameter that controls the dissimilarity between $\bx$ and $\bz$. This data augmentation \eqref{split_dist} is approximate in the sense that the marginal distribution
    \begin{eqnarray}
       \pi_{\rho}(\bx) &=& \int_{\mathbb{R}^n} \pi_{\rho}(\bx, \bz) d\bz \label{split_marg_dist} \\
       &\propto& \int_{\mathbb{R}^n}\exp \left[ - f(\bx,\by) - \beta g(\bz) - \dfrac{1}{2\rho^2}|| \bx - \bz||^2 \right] \dd \bz \nonumber
    \end{eqnarray}
coincides with the target distribution $\pi(\bx)$ only in the limiting case $\rho \rightarrow 0$. The conditional distributions\footnote{The conditional distributions associated to $\pi_{\rho}(\bx, \bz)$ are $p(\bx|\by,\bz; \rho^2)$ and $p(\bz | \bx;\rho^2)$. To lighten the notations, the coupling parameter $\rho^2$ will be omitted in what follows, i.e., one employs $p(\bx|\by,\bz)$ and $p(\bz | \bx)$.} associated to the augmented posterior $\pi_{\rho}(\bx, \bz)$  are given by
\begin{align}
p(\bx|\by,\bz) &\propto \exp \left[ -  f(\bx,\by) - \frac{1}{2\rho^2}|| \bx - \bz||^2 \right]
    \label{eq:condx} \\
    p(\bz | \bx) &\propto \exp \left[ - \beta g(\bz) - \dfrac{1}{2\rho^2}||\bx - \bz||^2 \right].
    \label{eq:condz}
\end{align}
The so-called split Gibbs sampler (SGS) alternatively samples according to these two conditional distributions to generate samples asymptotically distributed according to \eqref{split_dist} \cite{vono2019split,rendell2020global}. Interestingly, this splitting allows the two terms $f(\cdot,\by)$ and $g(\cdot)$ defining the full potential to be dissociated and involved into two distinct conditional distributions. SGS shares strong similarities with ADMM and HQS methods and is expected to lead to simpler, scalable and more efficient sampling schemes. 

Specifically, sampling according to the conditional posterior \eqref{eq:condx} can be interpreted as solving the initial estimation problem defined by the likelihood function \eqref{eq:likelihood} with now a Gaussian distribution with mean $\bz$ and diagonal covariance matrix $\rho^2 \mathbf{I}_n$ assigned as a prior. As stated earlier, a large family of imaging inverse problems, such as deblurring, inpainting and super-resolution is characterized by the quadratic potential function $f(\bx, \by)= \frac{1}{2\sigma^2}\|\bA \bx - \by \|^2_2$ also considered in this work. This leads to the Gaussian conditional distribution
\begin{equation}
    p(\bx|\by, \bz) = \mathcal{N}(\bx ; \boldsymbol{\mu}(\bz), \bQ^{-1})
\end{equation}
where the precision matrix $\bQ$ and the mean vector $\boldsymbol{\mu}(\cdot)$ are given by
\begin{equation}\label{eq:def_mu_and_Q}
    \left \{
   \begin{array}{r c l}
      \bQ &=& \dfrac{1}{\sigma^2} \bA^\top \bA + \dfrac{1}{\rho^2} \bI \\
      \boldsymbol{\mu}(\bz) &=& \bQ^{-1} \left( \dfrac{1}{\sigma^2} \bA^\top \by + \dfrac{1}{\rho^2} \bz  \right).
   \end{array}
   \right.
\end{equation}
In this case, sampling according to this conditional distribution can be efficiently achieved using dedicated algorithms that depend on the structure of the precision matrix $\bQ$. Interested readers are invited to consult \cite{vono2022high} for a recent overview of these methods. It is worth noting that when the potential function $f(\cdot, \by)$ is not quadratic, the proposed framework can embed proximal Monte Carlo algorithms to sample from \eqref{eq:condx}, as in \cite{Vono2018mlsp,Vono2019icassp}.

In the specific case considered in this work where $g(\cdot) = g_{\textrm{red}}(\cdot)$, the conditional distribution \eqref{eq:condz} can be interpreted as the posterior distribution associated to a Bayesian denoising problem equipped with a RED prior. Its objective boils down to inferring an object $\bz$ from the observations $\bx$ contaminated by an additive white Gaussian noise with a covariance matrix $\rho^2 \bI_n$. Sampling according to this conditional is not straightforward, in particular due to the regularization potential $g_{\textrm{red}}(\cdot)$ whose definition involves the denoiser $\mathsf{D}_{\nu}(\cdot)$. This work proposes to take advantage of the property \eqref{grad_g} by sampling from \eqref{eq:condz} following a Langevin Monte Carlo (LMC) step, i.e., 
\begin{equation}\label{eq:LMC_z}
    \bz^{(t+1)} = \bz^{(t)} + \gamma \nabla \log\, p\left(\bz^{(t)}\mid\bx\right) + \sqrt{2\gamma} \boldsymbol{\varepsilon}^{(t)}
\end{equation}
where $\gamma > 0$ is a fixed step-size and $\left\{\boldsymbol{\varepsilon}^{(t)}\right\}_{t\in \mathbb{N}}$ is a sequence of independent and identically distributed $n$-dimensional standard Gaussian random variables. Given the particular form of the conditional distribution \eqref{eq:condz} and the property \eqref{grad_g}, this recursion writes explicitly as
\begin{align}\label{LMC_step}
    \bz^{(t+1)} = &\left(1 - \gamma \beta +\frac{1}{\rho^2} \right) \bz^{(t)} \nonumber \\
    &-\frac{1}{\rho^2} \bx^{(t)} + \gamma \beta \mathsf{D}_{\nu}\left(\bz^{(t)}\right) + \sqrt{2\gamma} \boldsymbol{\varepsilon}^{(t)}.
\end{align}

The proposed so-called Langevin-within-SGS (LwSGS) instantiated to sample according to the RED posterior \eqref{eq:target} is summarized in Algo~\ref{alg:RED-LwSGS}.

\newcommand{\algocomment}[1]{ \STATEx {\color[rgb]{0.5,0.8,0.5}{\% \textit{#1}}}}
\begin{algorithm}[H]
\caption{LwSGS to sample from the RED posterior}
\label{alg:RED-LwSGS}
\begin{algorithmic}[1]
\renewcommand{\algorithmicrequire}{\textbf{Input:}}
\REQUIRE denoiser $\mathsf{D}_{\nu}(\cdot)$, regularization parameter $\beta$, coupling parameter $\rho^2$, step-size $\gamma$, number of burn-in iterations $N_{\mathrm{bi}}$, total number of iterations $N_{\mathrm{MC}}$
\renewcommand{\algorithmicrequire}{\textbf{Initialization:}}  
\REQUIRE $\bx^{(0)}$, $\bz^{(0)}$
\FOR{$t = 0$ to $N_{\mathrm{MC}}-1$}
    \algocomment{Sampling the splitting variable according to \eqref{LMC_step}}
    \STATE $\boldsymbol{\varepsilon}^{(t)} \sim \mathcal{N}(\boldsymbol{0},\mathbf{I})$
    \STATE $\bz^{(t+1)} = \left(1 - \gamma \beta +\frac{1}{\rho^2} \right) \bz^{(t)} -\frac{1}{\rho^2} \bx^{(t)} + \gamma \beta \mathsf{D}_{\nu}\left(\bz^{(t)}\right) + \sqrt{2\gamma} \boldsymbol{\varepsilon}^{(t)}$  
    \algocomment{Sampling the variable of interest according to \eqref{eq:condx}}
    \STATE  $\bx^{(t+1)} \sim \mathcal{N}\big( \boldsymbol{\mu}(\bz^{(t+1)}), \mathbf{Q}^{-1} \big)$
\ENDFOR
\renewcommand{\algorithmicensure}{\textbf{Output:}}
\ENSURE collection of samples $\{\bx^{(t)}, \bz^{(t)}\}_{t=N_{\mathrm{bi}+1}}^{N_{\mathrm{MC}}}$
\end{algorithmic}
\end{algorithm}

Because of the discretization followed by the LMC step \eqref{eq:LMC_z}, the samples produced by \eqref{LMC_step} are biased and are not exactly distributed according to \eqref{eq:condz}. To mitigate this bias and  ensure that LMC exactly targets \eqref{eq:condz}, one well-admitted strategy consists in including a Metropolis-Hasting (MH) step, resulting in the Metropolis adjusted Langevin algorithm \cite{roberts2002langevin}. Then, combined with the sampling according to \eqref{eq:condx}, the overall resulting sampling algorithm would become a canonical instance of Metropolis-within-Gibbs algorithm whose samples would be ensured to be distributed according to the augmented posterior distribution \eqref{split_dist}. However, performing this MH step within each iteration of the SGS requires to compute multiple corresponding MH ratios and to accept or reject the proposed samples, which may significantly increase the computational burden of the SGS. In this work, one proposes to bypass this MH correction, yet at the price of an approximation which is controlled. Indeed, the bias induced by the use of a LMC step within a SGS iteration will be investigated in the theoretical analysis conducted in  Section \ref{subsec:analysis}.

\subsection{Related Monte Carlo algorithms}
The proposed LwSGS algorithm shares some similarities with some recently developed Monte Carlo algorithms, mainly motivated by the will of conducting distributed Bayesian inference over several computer  nodes. In \cite{thouvenin2022distributed}, the authors adopt a splitting strategy offered by AXDA to derive a distributed SGS (DSGS) when the posterior distribution comprises multiple composite terms. As with LwSGS, the core idea of DSGS can be sketched as replacing the exact sampling of one conditional distribution of the augmented posterior distribution by a more efficient surrogate sampling technique. Given the particular form of the posterior distributions considered in \cite{thouvenin2022distributed}, a suitable choice of this surrogate is shown to be one step of a proximal stochastic gradient Langevin algorithm (PSGLA) \cite{durmus2019analysis}. When this sampling is not corrected by a MH step, this leads to an inexact instance of SGS coined as PSGLA-within-SGS. The synchronous distributed version of PSGLA-within-SGS accounts for the hypergraph structure of the involved composite terms to efficiently distribute the variables over multiple workers under controlled communication costs. In the same vein, another synchronous distributed MCMC algorithm referred to as DG-LMC has been introduced in \cite{plassier2021dg} to conduct Bayesian inference when the target log-posterior also writes as a sum of multiple composite terms. Also leveraging AXDA, it adopts a splitting scheme different from LwSGS and PSGLA-within-SGS. Yet, it can be interpreted as an inexact SGS for which multiple conditional distributions are approximately sampled thanks to LMC steps. While the existence of a stationary distribution targeted by PSGLA-within-SGS and its convergence have not been demonstrated in \cite{thouvenin2022distributed}, such a thorough theoretical analysis has been conducted for DG-LMC in  \cite{plassier2021dg}. Even if these results are a precious asset to conduct a similar analysis of LwSGS, they should be carefully adapted to fit the splitting scheme adopted by LwSGS.

\subsection{Theoretical analysis}\label{subsec:analysis}
This section provides theoretical insights regarding the proposed LwSGS algorithm.  For sake of generality and unless otherwise stated, this analysis is conducted for any regularization potential $g(\cdot)$ satisfying assumptions introduced below. However, these assumptions will be also discussed and examined under the prism of the RED paradigm, i.e., with $g(\cdot) = g_{\textrm{red}}(\cdot)$ and when the proposed LwSGS aims at targeting the RED posterior distribution \eqref{eq:target}.

As stated above, because of the absence of MH correction after the LMC step, Algo.~\ref{alg:RED-LwSGS} does not fall into the class of Metropolis-within-Gibbs samplers. Thus, the primer objective of this analysis is to demonstrate that the samples produced by  Algo.~\ref{alg:RED-LwSGS} are asymptotically distributed according to a unique invariant distribution\footnote{With a slight abuse of notation, one uses the same
notations for a probability distribution and its associated pdf.} $\pi_{\rho, \gamma}$ following an ergodic transition kernel denoted $P_{\rho, \gamma}$. Thanks to an appropriate synchronous coupling, the convergence analysis of LwSGS reduces to that of the Markov chain produced by the sampling \eqref{LMC_step} according to the conditional distribution $p(\bz|\bx)$. 
One first introduces and discusses two assumptions regarding the regularization potential $g(\cdot)$.

\begin{assumption}[Twice differentiability]
  \label{ass:twice_differentiable}
    The potential function $g(\cdot)$ is twice continuously differentiable and  there exists $M_g >0$ such that $\forall \bz \in \mathbb{R}^n$, $\|\nabla^2 g(\bz)\| \le M_g$.
  \end{assumption}
  
As stated in Section \ref{subsec:RED}, under the RED conditions \ref{RedCond:C1} and \ref{RedCond:C3}, the gradient of the RED potential is given by \eqref{grad_g}. This implies that the regularization potential $g_{\textrm{red}}(\cdot)$ is twice continuously differentiable with Hessian matrix $\nabla^2 g_{\textrm{red}}(\bz) = \bI_n - \nabla \mathsf{D}_{\nu}(\bz)$. Moreover, thanks to the Jacobian symmetry \ref{RedCond:C3} and strong passivity \ref{RedCond:C4} conditions, one has for all $\bz \in \mathbb{R}^n$, $\|\nabla^2 g_{\textrm{red}}(\bz)\| \le 2$. In other words, Assumption \ref{ass:twice_differentiable} always holds for RED.


\begin{assumption}[Strong convexity]
  \label{ass:supp_fort_convex}
    The  potential function $g(\cdot)$ is $m_g$-strongly convex, i.e., there exists $m_g >0$ such that $m_g \mathbf{I}_{n} \preceq \nabla ^2 g$.
  \end{assumption}
  
In the RED framework, a sufficient condition for the strong convexity of the potential $g_{\textrm{red}}(\cdot)$ is to ensure that the denoiser $\mathsf{D}_{\nu}(\cdot)$ is contractive, i.e., $\forall (\bz_1, \bz_2) \in \mathbb{R}^n \times \mathbb{R}^n$, $\| \mathsf{D}_{\nu}(\bz_1) - \mathsf{D}_{\nu}(\bz_2) \|_2 \le \epsilon \| \bz_1 - \bz_2 \|_2 $ for some Lipschitz constant $\epsilon < 1$. Under this condition, $g_{\textrm{red}}(\cdot)$ can be shown to be $m_g$-strongly convex 
with $m_g = 1-\epsilon > 0$. 
Unfortunately, most existing denoisers do not follow this contraction property \cite{hurault2022proximal}. To ensure the strong convexity of the potential $g_{\textrm{red}}(\cdot)$ when using a deep network-based denoiser, one solution would consist in explicitly including a regularization term into the training loss which constrains the Lipschitz constant \cite{ryu2019plug, terris2020building}. Finally, it is worth noting that  when Assumption~\ref{ass:supp_fort_convex} is satisfied, Assumption~\ref{ass:well_defined_density} is also satisfied, which implies that $p_{\textrm{red}}(\cdot)$ is well-defined.

Under these assumptions, the convergence of the proposed LwSGS algorithm is stated in the following proposition.

\begin{proposition}\label{prop:convergence}
     Let $\gamma \in \mathbb{R}_+^*$ such that $\gamma \le (\beta M_g + 1/\rho^2)^{-1}$.
          Then, under Assumptions~\ref{ass:twice_differentiable} and \ref{ass:supp_fort_convex}, the kernel $P_{\rho,\gamma}$ admits a unique stationary distribution $\pi_{\rho,\gamma}$. Moreover, for any $\bv = (\bx,\bz)^{\top} \in  \mathbb{R}^n \times  \mathbb{R}^n$ and any $t \in \mathbb{N}^*$, we have
   \begin{align*}
      W_{2}^2(\delta_{\bv} &P_{\rho,\gamma}^{t}, \pi_{\rho,\gamma}) \le
      C_1 \bigr(1- \gamma \beta m_g\bigr)^{2(t-1)} W_{2}^2(\delta_{\bv}, \pi_{\rho,\gamma}),
    \end{align*}
    where $C_1=1 + \frac{1}{\rho^2} \| \bQ^{-1} \|^2$.
\end{proposition}
\begin{proof}
    See Appendix \ref{proof_conv}.
\end{proof}

The rate of convergence of the proposed sampler is given by $1 - \gamma \beta m_g$. The asymptotic convergence of the samples produced by the kernel $P_{\rho,\gamma}$ towards the distribution $\pi_{\rho,\gamma}$ is only possible if the LMC discretization step $\gamma$ is sufficiently small. The formula above establishes convergence for any step-size $\gamma \le (\beta M_g + 1/\rho^2)^{-1}$. From Assumptions~\ref{ass:twice_differentiable} and \ref{ass:supp_fort_convex}, $- \log\, p(\bz|\bx)$ is $\beta m_g$-strongly convex and ($\beta M_g + 1/\rho^2$)-smooth, i.e., $\beta m_g \mathbf{I}_{n} \preceq - \nabla^2 \log \, p(\bz | \bx) \preceq (\beta M_g + 1/\rho^2) \mathbf{I}_{n}$. From \cite{plassier2021dg} and \cite{durmus2019high}, a sufficient condition on the step size $\gamma$ to ensure contraction w.r.t.~the Wasserstein distance is $\gamma \le 2/( \beta m_g + \beta M_g + 1/\rho^2)$. Thus the coupling parameter $\rho$ implicitly determines the convergence rate. The smaller $\rho$, the smaller $\gamma$ and the slower the convergence is.

Once the asymptotic convergence of the samples produced by Algo.~\ref{alg:RED-LwSGS} has been ensured, the second stage of the theoretical analysis consists in analyzing the bias between the stationary distribution $\pi_{\rho,\gamma}$ and the targeted augmented distribution $\pi_{\rho}$. This bias, which results from the use of the LMC step in Algo.~\ref{alg:RED-LwSGS} to target the conditional distribution $p(\bz | \bx)$ defined by \eqref{eq:condz}, is quantified in the following proposition.

\begin{proposition}\label{prop:bias}
Let $\gamma \in \mathbb{R}_{+}^{*}$ such that  
$\gamma \le 2(\beta m_g + \beta M_g + 1/\rho^2)^{-1}$.
Then, under Assumptions~\ref{ass:twice_differentiable} and \ref{ass:twice_differentiable}, we have
\begin{align*}
W_2^{2} (\pi_{\rho}, \pi_{\rho,\gamma})
& \le  n\gamma C_2 \tilde{M}^2 \Big(1 + \frac{\gamma^2\tilde{M}^2}{12}+ \frac{\gamma \tilde{M}^2}{2\tilde{m}}\Big),
\end{align*}
where $\tilde{m} =\beta m_g + 1/\rho^2$, $\tilde{M} =  \beta M_g + 1/\rho^2$ and $C_2=\frac{2}{\beta m_g}{(1 + \frac{1}{\rho^2} \| \bQ^{-1}\|^2) }$.
\end{proposition}
\begin{proof}
    See Appendix \ref{proof_bias}.
\end{proof}

The bias is upper-bounded by a term driven by the step size. More precisely, as $\gamma$ decreases towards zero, the squared Wasserstein distance $W_2^2(\pi_{\rho}, \pi_{\rho,\gamma})$ is upper-bounded  by $\mathcal{O}(n\gamma)$. Thus, for a sufficiently small step size $\gamma$, LwSGS produces a Markov chain with a stationary distribution $\pi_{\rho,\gamma}$ that can be arbitrary close to $\pi_{\rho}$. This result is in agreement with those stated in \cite[Proposition 4]{plassier2021dg} and \cite[Corollary 7]{durmus2019high}.

\subsection{Revisiting PnP-ULA and AXDA from the RED paradigm}\label{subsec:discussion_PnP}

This section draws connections between AXDA, the proposed algorithm and PnP-ULA \cite{laumont2022bayesian}. As a reminder, in a nutshell, PnP-ULA targets a posterior distribution of the form
\begin{equation}\label{eq:PnP_ULA_post}
    p_{\epsilon}(\bx|\by) \propto \exp\left[-f(\bx,\by)\right] p_{\epsilon}(\bx)
\end{equation}
with $\epsilon>0$ where \begin{equation}
    p_{\epsilon}(\bx) \propto \int_{\mathbb{R}^n} p(\bz) \exp\left[-\frac{1}{2\epsilon} \left\|\bx-\bz\right\|^2\right] \dd \bz
\end{equation}
defines the regularized counterpart of the prior distribution $p(\bx)$. Interestingly, when the prior writes $p(\bx) \propto \exp\left[-\beta g(\bx)\right]$ and $\rho^2 = \epsilon$, the posterior distribution $p_{\epsilon}(\bx|\by)$ targeted by PnP-ULA \eqref{eq:PnP_ULA_post} perfectly matches the marginal distribution $\pi_{\rho}(\bx)$ in \eqref{split_marg_dist} resulting from an AXDA strategy and targeted by a SGS.

Besides, from an algorithmic point of view, ULA recursions applied to \eqref{eq:PnP_ULA_post} conventionally write
\begin{equation*}\label{eq:ULA}
   \bx^{(t+1)} = \bx^{(t)} - \gamma \nabla f(\bx^{(t)} ,\by) + \gamma \nabla \log\, p_{\epsilon}(\bx^{(t)}) + \sqrt{2\gamma} \boldsymbol{\varepsilon}^{(t)}.
\end{equation*}
Thanks to the Tweedie's identity \cite{efron2011tweedie}, the score function can be replaced  by the denoising residual, i.e., $\epsilon \nabla \log\, p_{\epsilon}(\bx) = \mathsf{D}_{\epsilon}^*(\bx)-\bx$ where $\mathsf{D}_{\epsilon}^{*}(\cdot)$ is an MMSE denoiser. This leads to the (simplified) PnP-ULA scheme
\begin{eqnarray}
\text{(PnP-ULA)}:~\bx^{(t+1)}& =& \bx^{(t)} - \gamma \nabla f(\bx^{(t)} ,\by) \label{eq:sPnP_ULA}\\
&+& \frac{\gamma}{\epsilon} \left(\mathsf{D}_{\epsilon}^*(\bx^{(t)})-\bx^{(t)}\right) + \sqrt{2\gamma} \boldsymbol{\varepsilon}^{(t)}. \nonumber
\end{eqnarray}
This simplified scheme departs from the canonical PnP-ULA scheme studied in \cite{laumont2022bayesian} by only omitting an additional term $\frac{\gamma}{\lambda} \left[\Pi_{\mathbb{S}}(\bx^{(t)})-\bx^{(t)}\right]$ where $\Pi_{\mathbb{S}}(\cdot)$ denotes the projection onto the convex and compact set $\mathbb{S}$. This term has been included into the PnP-ULA scheme for technical reasons to derive convergence results. Its impact will be empirically shown to be marginal in practice (see Section \ref{subsec:PnPvsRED}). Conversely, using \eqref{grad_g}, ULA recursions applied to the RED posterior \eqref{eq:target} writes
\begin{eqnarray}
\text{(RED-ULA)}:~\bx^{(t+1)} &=& \bx^{(t)} - \gamma \nabla f(\bx^{(t)} ,\by) \label{eq:RED_ULA}\\
&+& {\gamma}\beta \left(\mathsf{D}_{\nu}(\bx^{(t)})-\bx^{(t)}\right) + \sqrt{2\gamma} \boldsymbol{\varepsilon}^{(t)}. \nonumber
\end{eqnarray}
It clearly appears that the RED-ULA scheme defined by the previous recursion \eqref{eq:RED_ULA} coincides with PnP-ULA \eqref{eq:sPnP_ULA}  when $\beta=\frac{1}{\epsilon}$ and the denoiser embedded into RED is chosen as the MMSE denoiser, i.e., $\mathsf{D}_{\nu}(\cdot) = \mathsf{D}_{\epsilon}^*(\cdot)$. Since in practice the use of an MMSE denoiser is infeasible, PnP-ULA is implemented with an off-the-shelf denoiser. Thus, RED-ULA is no more than the practical implementation of PnP-ULA.

Moreover, an important corollary deals with the respective prior distributions defining the posteriors targeted by the three considered algorithms, namely PnP-ULA, RED-ULA and RED-SGS\footnote{The algorithmic scheme denoted RED-SGS is defined as a canonical Gibbs sampler with targets the augmented distribution \eqref{split_dist}. It can be interpreted as the RED-LwSGS for which the sampling according to \eqref{eq:condz} would be exact.}. Indeed, on one hand, PnP-ULA and RED-ULA target the same posterior distribution \eqref{eq:PnP_ULA_post} (provided the use of an MMSE denoiser). On another hand, PnP-ULA and RED-SGS also target the same posterior distribution \eqref{eq:target}. It yields that the regularized prior implicitly induced by AXDA coincides with the RED prior based on an MMSE denoiser, i.e., $p_{\epsilon}(\bx) = p_{\textrm{red}}^*(\bx)$ and, in particular,
\begin{equation}
   - \log\, p_{\epsilon}(\bx) = \frac{1}{2\epsilon} \bx^{\top}[\bx-\mathsf{D}_{\epsilon}^*(\bx)].
\end{equation}
Note that differentiating this identity obviously leads again to the celebrated Tweedie's identity.

\section{Experiments}\label{sec:experiments}
\subsection{Experimental setup}
Experiments have been conducted based on two popular image data sets, namely the Flickr Faces High Quality (FFHQ) data set \cite{karras2019style} and the ImageNet data set \cite{deng2009imagenet}. All images are RGB images of size $256 \times 256$ pixels ($n=256^2$) and have been normalized to the range $[0,1]$. The performance of the proposed RED-LwSGS algorithm is assessed w.r.t.~to three inversion tasks
\begin{itemize}
    \item \emph{deblurring}: the operator $\bA$ is assumed to be a $n \times n$ circulant convolution matrix associated with a spatially invariant blurring kernel. It is chosen as a Gaussian kernel of size $25 \times 25$ with standard deviation $1.6$.
    \item \emph{inpainting}: the operator $\bA$ stands for a binary mask with $m \ll n$. It is designed such that $80\%$ of the total pixels are randomly masked accross the three color channels.
    \item \emph{super-resolution}: the operator $\bA$ is decomposed as $\bA = \bS \bB$ where the ${n \times n}$  matrix $\bB$   stands for a spatially invariant Gaussian blur of size $7 \times 7$ with standard deviation $1.6$ and the operator $\bS$ is a ${m \times n}$ binary matrix which performs a regular subsampling of factor $d=4$ in each dimension (i.e., $m = n d^2$). It is worth noting that for this task, the AXDA model needs to be adapted to account for the specific structure of the matrix $\bA$ (more details are given in Appendix \ref{sec:superresolution}).
\end{itemize}
For all tasks, the degraded images have been corrupted by an additive Gaussian noise to reach a signal-to-noise ratio of $\mathrm{SNR=}30$dB.

\subsection{Compared methods}

The proposed RED-LwSGS algorithm has been implemented using the pre-trained deep network DRUNet \cite{zhang2021plug} as the denoiser $\mathsf{D}_{\nu}(\cdot)$. It has been taken directly from the corresponding repository and has been used without further fine-tuning for the considered inversion tasks, see Appendix \ref{app:subsec:denoiser} for complementary information. The test images have never been seen by the model while training to avoid any bias due to potentially overfitted pre-trained models. The proposed method has been compared to several state-of-the-art inversion methods:
\begin{itemize}
    \item RED-ADMM \cite{romano2017little}: ADMM with RED regularization based on the same DRUNet denoiser;
    \item RED-HQS: HQS algorithm  with RED regularization based on the same DRUNet denoiser;
    \item PnP-ADMM \cite{venkatakrishnan2013plug}: ADMM with a PnP regularization based on the same DRUNet denoiser;
    \item PnP-ULA \cite{laumont2022bayesian}: ULA with a PnP regularization based on the same DRUNet denoiser;
    \item TV-SP \cite{vono2019split}: SGS with a TV regularization;
    \item TV-MYULA \cite{durmus2018efficient}: Moreau-Yosida ULA with TV regularization;
    \item DiffPIR \cite{zhu2023denoising}: denoising diffusion model for PnP image restoration.
\end{itemize}
It is worth noting that RED-ADMM, RED-HQ and  PnP-ADMM are optimization-based methods which provide only point estimates of the restored images. Conversely, PnP-ULA, TV-SP and TV-MYULA are Monte Carlo methods that are able to enrich point estimates with credibility intervals. For these sampling methods, the results reported below correspond to the MMSE estimates approximated by averaging the generated samples after the burn-in period. Complementary information regarding the algorithm implementations are reported in Appendices \ref{app:impl_detail_LwSGS} and \ref{app:compared methods}.

\subsection{Figures-of-merit}
Beyond performing visual inspection, the methods are compared with respect to several quantitative figures-of-merit. Peak signal noise ratio (PSNR) (dB) and structural similarity index (SSIM) are considered as image quality metrics (the higher the score, the better the reconstruction). They are complemented with a perceptual metric, namely the learned perceptual image patch similarity (LPIPS), for which the lower the score, the better the reconstruction. Moreover, to assess the effectiveness of the sampling algorithms studied, they are also compared in term of integrated autocorrelation time (IAT), which is also an indicator for good or bad mixing (the lower, the better) \cite{Fall2014}. Finally, all methods are compared in terms of computational times when the algorithms are implemented on a server equipped with 48 CPU cores Intel 2.8Ghz, 384Go RAM, Nvidia A100 GPU. 

\subsection{Experimental results}
\label{subsec:results}
The results obtained by the compared algorithms when performing the three considered tasks are reported in Tables \ref{tab:Ffhq-metrics} and \ref{tab:Imagenet-metrics} for the two datasets FFHQ and ImageNet, respectively. These results show that the proposed RED-LwSGS method  achieves very competitive performance for all three tasks. In the case of inpainting, the forward operator masking $80\%$ pixels is non-invertible and the problem is expected to require further prior regularization than the two other tasks. 
For this task, algorithms relying on data-driven regularizations, such as RED-LwSGS and PnP-ULA, appear to include more informative priors when compared to TV-MYULA and TV-SP which rely on the same model-based regularization. For super-resolution, which is a more challenging problem than deblurring and inpainting, RED-LwSGS performs similarly to DiffPIR, RED-ADMM and RED-HQS when other MCMC algorithms fail.

\begin{table*}
\setlength{\tabcolsep}{2pt}
\centering
\caption{FFHQ data set: average performance over a test set of 100 images and corresponding standard deviations. \textbf{Bold}: best score, \underline{underline}: second score.\label{tab:Ffhq-metrics}}
\label{Tab:perfs_ffhq}
\begin{tabular}{lcccccccccc} 
\midrule
    & &\textbf{Observation} &\textbf{RED-LwSGS} &\textbf{PnP-ULA}  &\textbf{TV-MYULA}  &\textbf{TV-SP} &\textbf{RED-ADMM}  &\textbf{RED-HQS}  &\textbf{PnP-ADMM} &\textbf{DiffPIR}\\
\midrule
\multicolumn{1}{c}{ \multirow{5}{*}{\rotatebox{90}{\textbf{Deblurring}}}  } 
&\textbf{PSNR}(dB) $\uparrow$ & 29.8550 & 40.13±1.931 & 38.06±1.986 & 38.03±2.028 & 39.89±1.320 & \textbf{41.28}±2.459 & \underline{40.80}±2.560 & 34.03±7.601 & 37.15±2.066\\
&\textbf{SSIM} $\uparrow$ & 0.8994 & 0.982±0.003 & 0.975±0.006 & 0.981±0.006 &  0.970±0.004 & \textbf{0.985}±0.004 & \underline{0.984}±0.005 & 0.869±0.154 & 0.957±0.010 \\
&\textbf{LPIPS} $\downarrow$  & 0.0487 & 0.003±0.002 & 0.005±0.004 & 0.005±0.004 & \underline{0.002}±0.001 & \textbf{0.002}±0.002 & 0.002±0.002 & 0.033±0.080 & 0.007±0.001 \\
&\textbf{IAT} $\downarrow$   & - & \underline{75.17}±12.42 & 75.21±12.43 & 75.22±12.42 & \textbf{75.14}±12.42 & & & & \\
&\textbf{Time}(s) $\downarrow$ & - & 44±5 & 52±7 &  64±5 & 40±5 & 2±0 & 2±0 &  2±0 & 1±0 \\
\midrule
\multicolumn{1}{c}{ \multirow{5}{*}{\rotatebox{90}{\textbf{Inpainting}}}  } 
&\textbf{PSNR}(dB) $\uparrow$ & 7.2069 & 30.73±2.932 & \underline{31.46}±2.650 & 27.71±1.881 & 27.27±1.781 & \textbf{31.63}±2.672 & 31.36±2.293 & 31.32±3.142 & 31.26±2.25 \\
&\textbf{SSIM} $\uparrow$ & 0.0678 & 0.908±0.028 & 0.906±0.026 &  0.830±0.040 & 0.815±0.041 & \underline{0.911}±0.024 & 0.901±0.025 & \textbf{0.915}±0.042 & 0.890±0.025 \\
&\textbf{LPIPS} $\downarrow$  & 0.5831 & 0.023±0.017 & \underline{0.020}±0.013 & 0.056±0.026 & 0.061±0.027 & \textbf{0.019}±0.013 & 0.021±0.014 & \textbf{0.019}±0.015 & 0.021±0.005 \\
&\textbf{IAT} $\downarrow$   & - & \underline{75.24}±12.48 & \underline{75.51}±12.55 & 76.97±13.51 & 113.86±12.17 & - & - & - & - \\
&\textbf{Time}(s) $\downarrow$ & - & 74±1 & 79±1 & 150±7 & 71±1 & 3±0 & 2±0 & 2±0 & 2±0 \\
\midrule
\multicolumn{1}{c}{ \multirow{5}{*}{\rotatebox{90}{\textbf{Super-res.}}}  } 
&\textbf{PSNR}(dB) $\uparrow$  & - & 30.43±2.161 & 29.01±2.013 & 28.99±2.017 & 28.94±2.019 & 30.49±2.222 & \underline{30.54}±2.206 & 30.13±2.184 &  \textbf{30.99}±2.212 \\
&\textbf{SSIM} $\uparrow$ & - & 0.872±0.036 & 0.847±0.037 & 0.847±0.037 & 0.846±0.037 & \underline{0.875}±0.036 & \textbf{0.876}±0.036 & 0.867±0.037 & 0.868±0.034\\
&\textbf{LPIPS} $\downarrow$ & - & 0.035±0.021 & 0.050±0.024 & 0.049±0.024 & 0.051±0.024 & \underline{0.034}±0.020 & \underline{0.034}±0.020 & 0.035±0.021 & \textbf{0.011}±0.008 \\
&\textbf{IAT} $\downarrow$   & - & \textbf{75.75}±12.62 & 75.88±12.64 & \underline{75.85}±12.64 & 75.86±12.63 & - & - & - & - \\
&\textbf{Time}(s) $\downarrow$  & - & 115±25 & 128±40 & 133±26 & 112±23 & 3±1 & 3±1 & 3±1 & 2±0 \\
\midrule
\end{tabular}
\end{table*}

\begin{table*}
\setlength{\tabcolsep}{2pt}
\centering
\caption{ImageNet data set: average performance over a test set of 100 images and corresponding standard deviations. \textbf{Bold}: best score, \underline{underline}: second score.\label{tab:Imagenet-metrics}}
\label{Tab:perfs_imagenet}
\begin{tabular}{lcccccccccc} 
\midrule
    & &\textbf{Observation} &\textbf{RED-LwSGS} &\textbf{PnP-ULA}  &\textbf{TV-MYULA}  &\textbf{TV-SP} &\textbf{RED-ADMM}  &\textbf{RED-HQS}  &\textbf{PnP-ADMM} &\textbf{DiffPIR}\\
\midrule
\multicolumn{1}{c}{ \multirow{5}{*}{\rotatebox{90}{\textbf{Deblurring}}}  } 
&\textbf{PSNR}(dB) $\uparrow$ & 26.6362 & \underline{34.85}±3.961 & 32.39±4.063 & 31.89±4.075 & 34.58±4.026 & \textbf{35.41}±4.024 & 34.48±3.062 & 31.63±7.783 & 31.89±4.139\\
&\textbf{SSIM} $\uparrow$ & 0.7676 & \underline{0.955}±0.024 & 0.925±0.046 & 0.921±0.051 & 0.954±0.026 & \textbf{0.956}±0.024 & 0.914±0.064 & 0.867±0.171 & 0.896±0.065 \\
&\textbf{LPIPS} $\downarrow$   & 0.1312 & 0.018±0.019 & 0.036±0.037 & 0.041±0.040 & 0.020±0.021 & \underline{0.016}±0.019 & \textbf{0.012}±0.011 & 0.059±0.108 & 0.044±0.016 \\
&\textbf{IAT} $\downarrow$   & - & \textbf{66.29}±19.98 & 66.76±19.89 & 66.89±19.89 & \underline{66.34}±19.98 & - & - & - &  - \\
&\textbf{Time}(s) $\downarrow$   & - & 44±3 & 50±2 & 69±32 & 43±1 & 2±1 & 2±1 & 2±1 & 1±0\\
\midrule
\multicolumn{1}{c}{ \multirow{5}{*}{\rotatebox{90}{\textbf{Inpainting}}}  } 
&\textbf{PSNR}(dB) $\uparrow$ & 7.3659 & 25.74±4.461 & 26.45±4.208 & 24.82±3.431 & 24.57±3.301 & \underline{26.52}±4.390 & 26.29±4.377 & 26.27±4.387 & \textbf{26.87}±3.97 \\
&\textbf{SSIM} $\uparrow$ & 0.0632 & 0.774±0.136 & 0.777±0.125 & 0.697±0.129 & 0.684±0.128 & \textbf{0.787}±0.119 & \underline{0.779}±0.120 & 0.769±0.132 & 0.765±0.123 \\
&\textbf{LPIPS} $\downarrow$ &  0.5824 & 0.097±0.082 & 0.088±0.078 & 0.136±0.084 & 0.144±0.084 & \underline{0.085}±0.074 & 0.090±0.081 &  0.089±0.079 & \textbf{0.035}±0.034 \\
&\textbf{IAT} $\downarrow$   & - & \textbf{66.53}±19.79 & \underline{67.99}±19.52 & 113.60±18.53 & 107.55±18.57 & - & - & - & - \\
&\textbf{Time}(s) $\downarrow$ & - & 75±2 & 80±3 & 163±14 & 73±1 & 3±0 & 3±0 & 3±0 & 8±0\\
\midrule
\multicolumn{1}{c}{ \multirow{5}{*}{\rotatebox{90}{\textbf{Super-res.}}}  } 
&\textbf{PSNR}(dB) $\uparrow$ & - & \underline{26.24}±3.933 & 25.42±3.461 & 25.35±3.445 & 25.39±3.445 & 26.22±3.936 & 26.21±3.926 & 25.84±3.899 & \textbf{26.31}±3.791 \\
&\textbf{SSIM} $\uparrow$ & - & \textbf{0.722}±0.148 & 0.693±0.142 & 0.687±0.146 & 0.693±0.141 & \underline{0.721}±0.148 & \underline{0.721}±0.148 & 0.707±0.150 & 0.712±0.142 \\
&\textbf{LPIPS} $\downarrow$ & - & \underline{0.120}±0.089 & 0.142±0.090 & 0.148±0.092 & 0.142±0.089 & \underline{0.120}±0.090 & \underline{0.120}±0.089 & 0.123±0.090 &  \textbf{0.106}±0.051 \\
&\textbf{IAT} $\downarrow$   & - & \textbf{93.36}±27.81 & 98.38±29.16 & 102.11±28.45 & \underline{97.25}±28.12 & - & - & - & - \\
&\textbf{Time}(s) $\downarrow$  & - & 112±22 & 119±30 & 129±15 & 109±13 & 3±0 & 3±0 & 3±0 & 3±0\\
\midrule
\end{tabular}
\end{table*}

Figures \ref{fig:ffhq} and \ref{fig:imagenet} visually assess the performance by depicting the results obtained on test images drawn from the FFHQ and ImageNet datasets. For all the considered tasks, the proposed RED-LwSGS method produces high-quality,  realistic images that closely match ground-truth details. As mentioned above, RED-LwSGS generates samples asymptotically distributed according to the posterior distribution. These samples can be used to quantify estimation uncertainty. The rightmost panels in Fig.~\ref{fig:ffhq} and \ref{fig:imagenet} illustrate this advantage by depicting the estimated pixelwise standard deviations obtained by the proposed algorithms. It is worth noting that this added value cannot be provided by optimization-based methods such as DiffPIR, RED-ADMM, RED-HQS and PnP-ADMM, which only offer point estimates. Noticeably,  pixels located in homogeneous regions are characterized by  lower uncertainty, while pixels in textured regions, edges, or complex structures appear to be estimated with more difficulty.

Tables \ref{tab:Ffhq-metrics} and \ref{tab:Imagenet-metrics} also reported the computational times required by the compared algorithms when tackling each restoration tasks. Among the class of sampling methods, RED-LwSGS stands out for its smallest computational times. Moreover, as for the other sampling-based methods, RED-LwSGS  remains within a factor of 50 compared to the optimization-based methods, namely DiffPIR, RED-ADMM, RED-HQS, and PnP-ADMM. The price to pay for offering an uncertainty quantification on top of point estimation seems reasonable.

\newcommand{\subplotwidth}{.16\textwidth}
\newcommand{\subplotwidthbis}{.17\textwidth}
\begin{figure*}
\centering
\begin{tabularx}{0.97\textwidth} { 
   >{\centering\arraybackslash}X 
   >{\centering\arraybackslash}X 
   >{\centering\arraybackslash}X 
   >{\centering\arraybackslash}X 
   >{\centering\arraybackslash}X 
   >{\centering\arraybackslash}X 
   }
    \footnotesize{Ground Truth} & \footnotesize{Observation} & \footnotesize{RED-ADMM} & \footnotesize{PnP-ULA} & \footnotesize{RED-LwSGS} & \footnotesize{RED-LwSGS (std)}
\end{tabularx}
    \includegraphics[width=\subplotwidth]{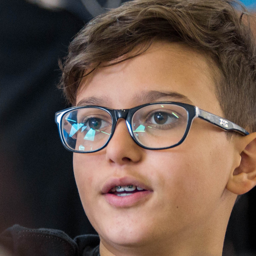}%
    \includegraphics[width=\subplotwidth]{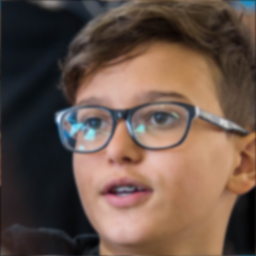}%
    \includegraphics[width=\subplotwidth]{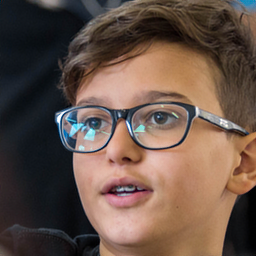}%
    \includegraphics[width=\subplotwidth]{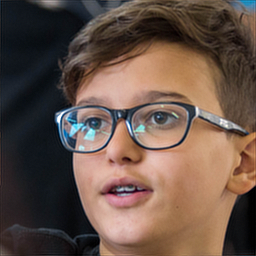}%
    \includegraphics[width=\subplotwidth]{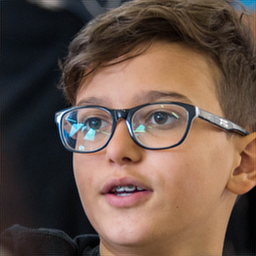}%
    \includegraphics[height=\subplotwidth,width=\subplotwidthbis]{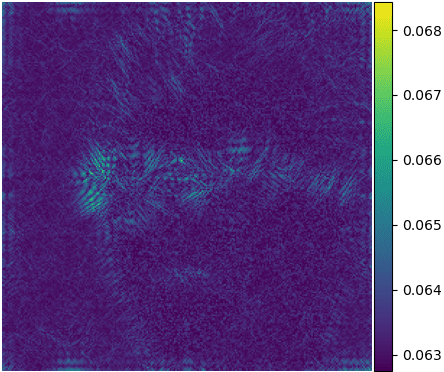}\\
    \includegraphics[width=\subplotwidth]{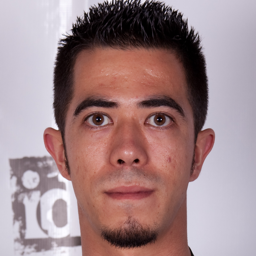}%
    \includegraphics[width=\subplotwidth]{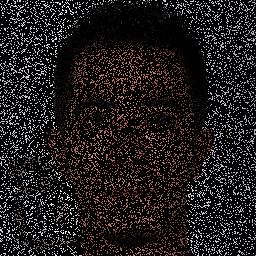}%
    \includegraphics[width=\subplotwidth]{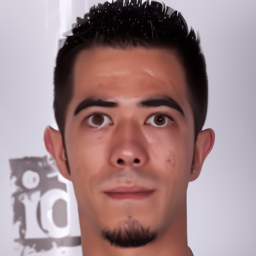}%
    \includegraphics[width=\subplotwidth]{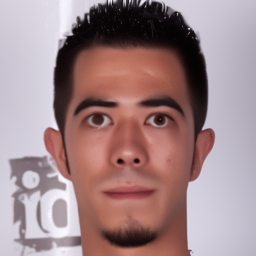}%
    \includegraphics[width=\subplotwidth]{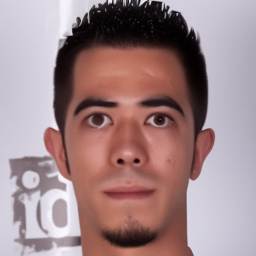}%
    \includegraphics[height=\subplotwidth,width=\subplotwidthbis]{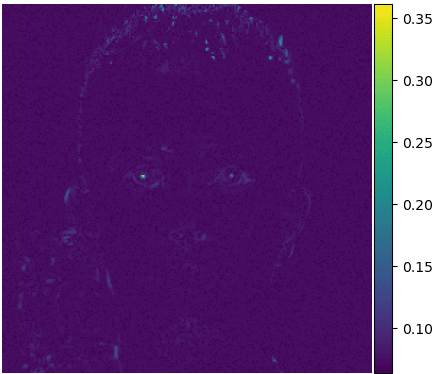}\\
    \includegraphics[width=\subplotwidth]{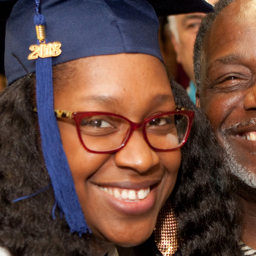}%
	\includegraphics[width=\subplotwidth]{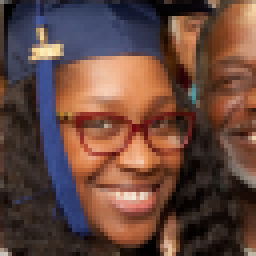}%
	\includegraphics[width=\subplotwidth]{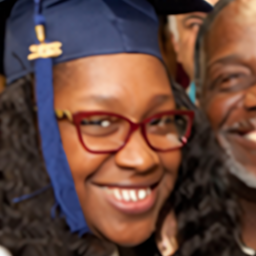}%
    \includegraphics[width=\subplotwidth]{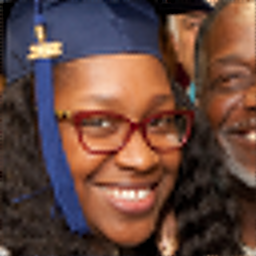}%
    \includegraphics[width=\subplotwidth]{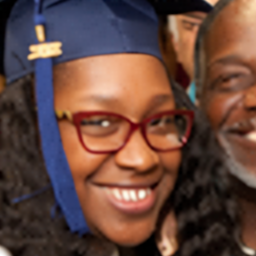}%
    \includegraphics[height=\subplotwidth,width=\subplotwidthbis]{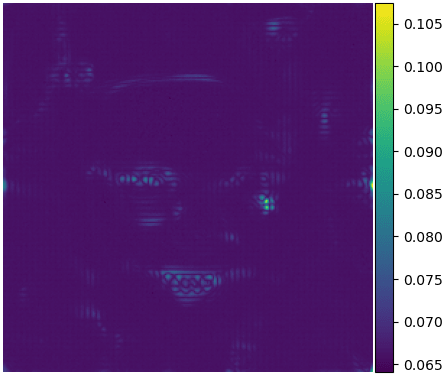}%
\caption{FFHQ data set: images recovered by the  compared methods for deblurring (top), inpainting (middle) and super-resolution (bottom).}
\label{fig:ffhq}
\end{figure*}

\begin{figure*}
\centering
\begin{tabularx}{0.97\textwidth} { 
   >{\centering\arraybackslash}X 
   >{\centering\arraybackslash}X 
   >{\centering\arraybackslash}X 
   >{\centering\arraybackslash}X 
   >{\centering\arraybackslash}X 
   >{\centering\arraybackslash}X 
   }
    \footnotesize{Ground Truth} & \footnotesize{Observation} & \footnotesize{RED-ADMM} & \footnotesize{PnP-ULA} & \footnotesize{RED-LwSGS} & \footnotesize{RED-LwSGS (std)}
\end{tabularx}
    \includegraphics[width=\subplotwidth]{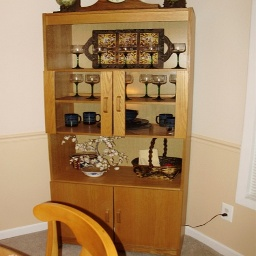}%
    \includegraphics[width=\subplotwidth]{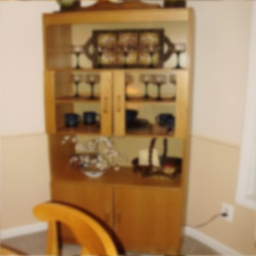}%
    \includegraphics[width=\subplotwidth]{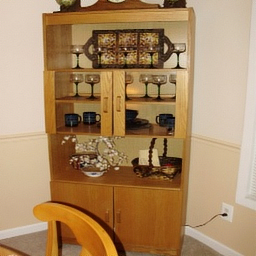}%
    \includegraphics[width=\subplotwidth]{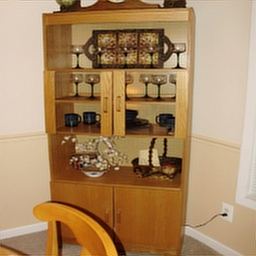}%
    \includegraphics[width=\subplotwidth]{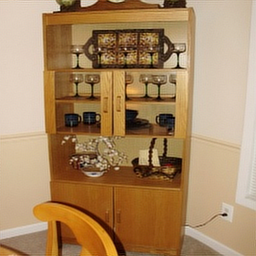}%
    \includegraphics[height=\subplotwidth,width=\subplotwidthbis]{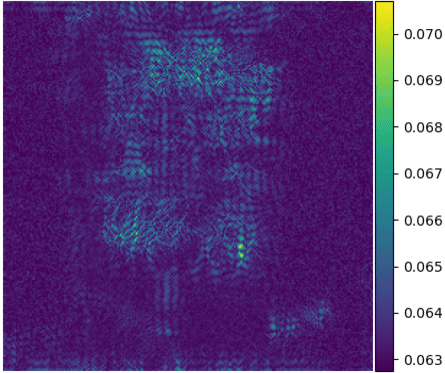}\\
    \includegraphics[width=\subplotwidth]{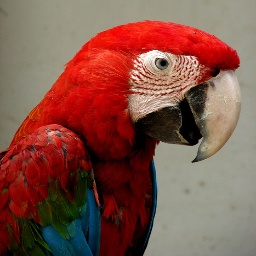}%
    \includegraphics[width=\subplotwidth]{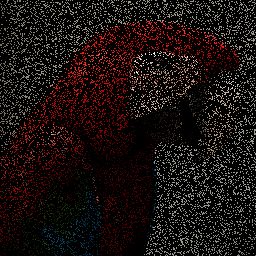}%
    \includegraphics[width=\subplotwidth]{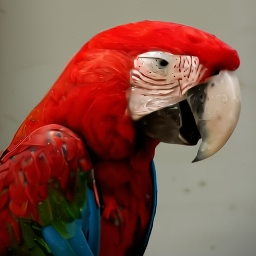}%
    \includegraphics[width=\subplotwidth]{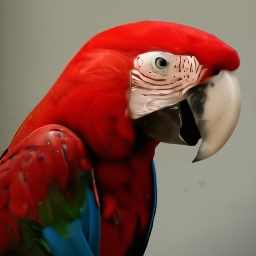}%
    \includegraphics[width=\subplotwidth]{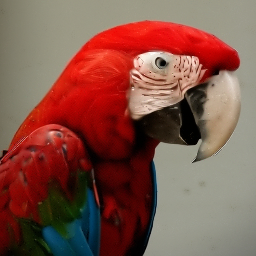}%
    \includegraphics[height=\subplotwidth,width=\subplotwidthbis]{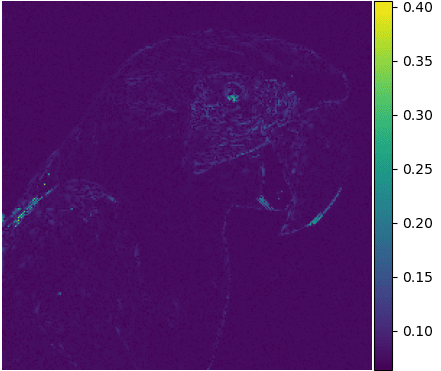}\\
    \includegraphics[width=\subplotwidth]{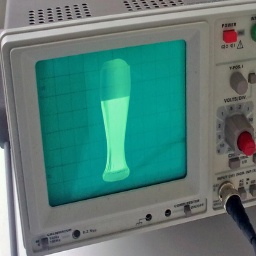}%
	\includegraphics[width=\subplotwidth]{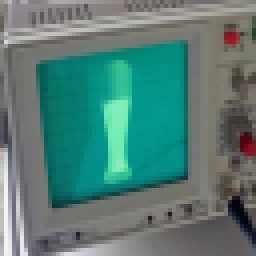}%
	\includegraphics[width=\subplotwidth]{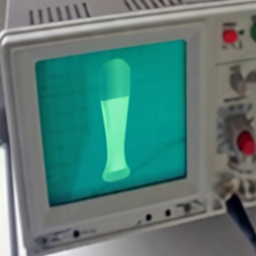}%
    \includegraphics[width=\subplotwidth]{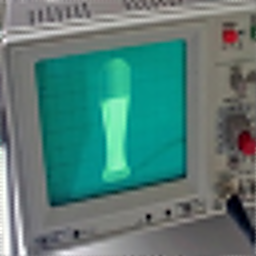}%
    \includegraphics[width=\subplotwidth]{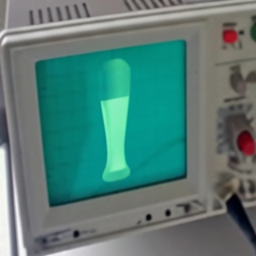}%
    \includegraphics[height=\subplotwidth,width=\subplotwidthbis]{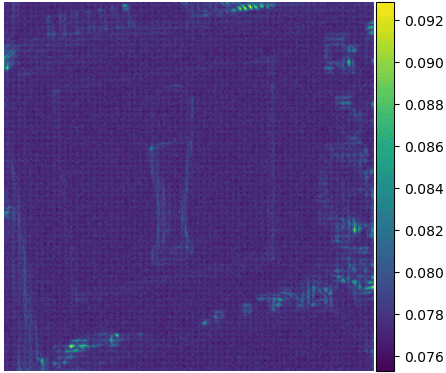}%
\caption{ImageNet data set: images recovered by the  compared methods for deblurring (top), inpainting (middle) and super-resolution (bottom).}
\label{fig:imagenet}
\end{figure*}

Finally, the convergence properties of the compared sampling-based algorithms have been assessed by monitoring the autocorrelation function (ACF) of the median components of the chains generated by those algorithms. By denoting $\bx^{(t)} = \left[x_1^{(t)},\ldots,x_n^{(t)}\right]^{\top}$, the median component  has been defined as the produced pixelwise Monte Carlo chain $\left\{x_{{i}}^{(t)}\right\}_{t=N_{\textrm{bi}}+1}^{N_{\mathrm{MC}}}$ with the median variance. Faster decreasing ACF means that the samples are less correlated and generally implies faster convergence of the Markov chain. Fig.~\ref{afc:ffhq} depicts these ACFs for the three restoration tasks conducted on one image from the FFHQ data set. For the deblurring task, it is not clear which of the compared methods is the more efficient, i.e., with the fastest ACF decay. Conversely, for the inpainting and super-resolution tasks, the ACF of RED-LwSGS decreases faster than the ACFs obtained with the two other Monte Carlo algorithms. This finding is confirmed by the IAT measures reported in Tables \ref{tab:Ffhq-metrics} and \ref{tab:Imagenet-metrics}.

\begin{figure*}
\centering
    \centering
    \includegraphics[width=.325\linewidth]{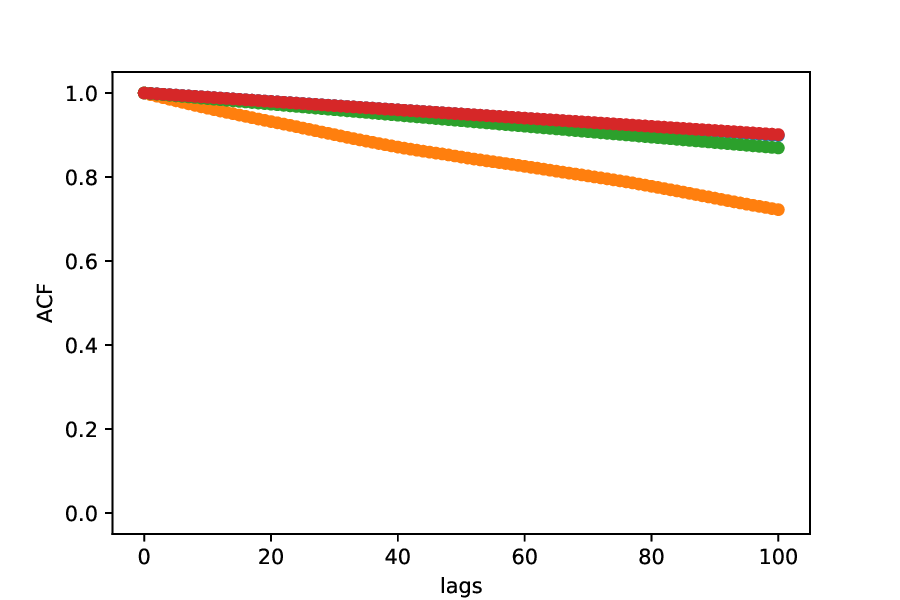}
    \includegraphics[width=.325\linewidth]{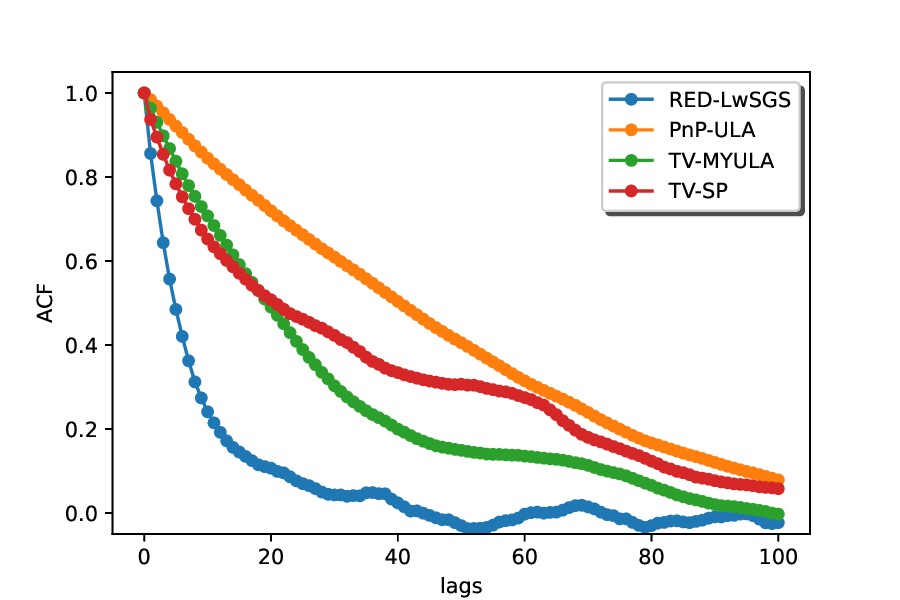}
    \includegraphics[width=.325\linewidth]{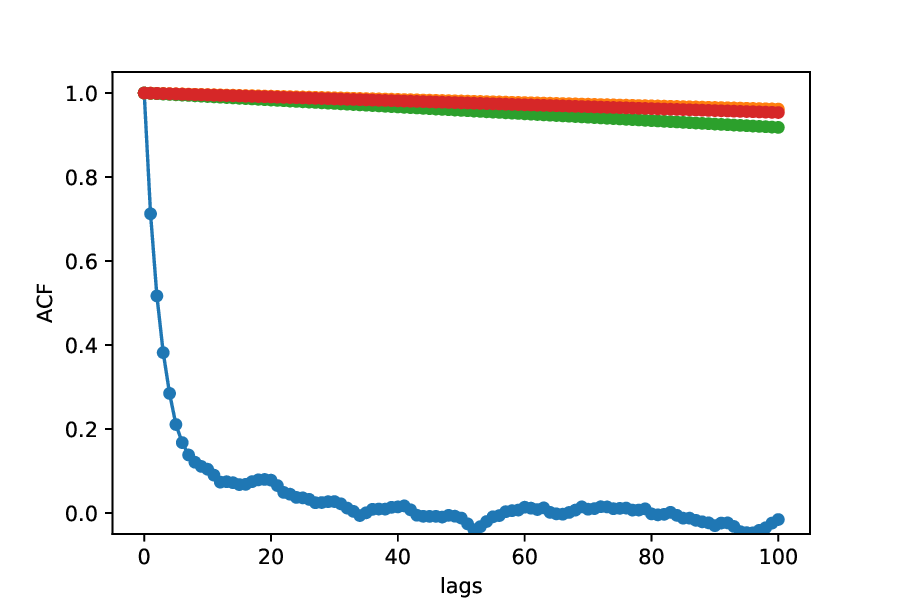}
\caption{FFHQ data set: (absolute) autocorrelation function (ACF) of the samples generated by the compared algorithms for deblurring (left), inpainting (middle) and super-resolution (right).}
\label{afc:ffhq}
\end{figure*}

\subsection{Does DRUNet meet the RED conditions?}
The explicit formula \eqref{grad_g} of the RED gradient requires the RED conditions \ref{RedCond:C1}--\ref{RedCond:C4} to be satisfied, i.e., the denoiser $\mathsf{D}_{\nu}(\cdot)$ should be differentiable, locally homogeneous, with symmetric Jacobian and strongly passive. It is legitimate to assess whether these conditions are verified for the deep denoiser considered in the experiments, namely DRUNet. First, the canonical implementation of DRUNet makes $\mathsf{D}_{\nu}(\cdot)$ not continuously  differentiable w.r.t.~the input due to the use of  ReLU activation functions. However, to ensure \ref{RedCond:C2}, the overall architecture can be made continuously differentiable by replacing them  with SoftPlus activation functions, which are $\mathcal{C}^{\infty}$, as suggested in \cite{hurault2022proximal}. Besides, the three other RED conditions are empirically assessed through numerical experiments. The local homogeneity condition \ref{RedCond:C1} of the denoiser is evaluated by computing the two following normalized  mean square errors \cite{reehorst2018regularization}
\begin{equation*}
    \mathrm{NMSE}_{{\mathrm{LH},1}} = \mathbb{E}\left[\dfrac{||\mathsf{D}_{\nu}\bigr( (1+\epsilon){\bx}\bigr) - (1+\epsilon) \mathsf{D}_{\nu}(\bx) ||^2} { ||(1+\epsilon) \mathsf{D}_{\nu}(\bx)||^2} \right]
\end{equation*}
and
\begin{equation*}
    \mathrm{NMSE}_{{\mathrm{LH},2}} = \mathbb{E}\left[ \dfrac{||\nabla \mathsf{D}_{\nu}({\bx})\bx - \mathsf{D}_{\nu}(\bx) ||^2} { ||\mathsf{D}_{\nu}(\bx)||^2} \right]
\end{equation*}
which are motivated by the definition \eqref{eq:LH} and the property \eqref{eq:homogeneity}, respectively. These metrics should be close to zero to ensure local homogeneity. Further, following \ref{RedCond:C3}, the denoiser should have a symmetric Jacobian. This characteristic  is empirically evaluated by computing 
\begin{equation*}
    \mathrm{NMSE}_{{\mathrm{JS}}} = \mathbb{E}\left[\dfrac{|| \nabla \mathsf{D}_{\nu}(\bx) - \nabla \mathsf{D}^\top_{\nu}(\bx) ||^2} { ||\nabla \mathsf{D}_{\nu}(\bx)||^2}\right]
    \end{equation*}
which should be close to zero for a symmetric Jacobian \cite{reehorst2018regularization}. Finally, to assess the strong passivity condition \ref{RedCond:C4}, one considers the mean spectral radius
\begin{equation*}
       {{\mathrm{MSR}}} = \mathbb{E}\left[\eta \left( \nabla \mathsf{D}_{\nu}(\bx) \right)\right]
\end{equation*}
which can be computed by the power iteration and should be smaller than 1. 
When computing these metrics, the $(i,j)$-element of the gradient of $\mathsf{D}_{\nu}(\cdot)$ has been approximated as
    \begin{equation*}
    [\nabla \mathsf{D}_{\nu}(\bx)]_{i,j} \approx \dfrac{ [\mathsf{D}_{\nu}(\bx + \epsilon {\mathbf{e}}_j)]_i - [\mathsf{D}_{\nu}(\bx - \epsilon \mathbf{e}_j)]_i }{2\epsilon}
    \end{equation*}
where $\mathbf{e}_j$ denotes the $j$th canonical basis vector, i.e., the $j$th column of $\mathbf{I}_{n}$, and $\epsilon > 0$ is small enough. 
These metrics are reported in Table \ref{tab:RED_conditions} when the four considered scores have been computed over $100$ patches of size $32 \times 32$ extracted from images of the two data sets, namely FFHQ and ImageNet. They show that the DRUNet denoiser seems to satisfy the local homogeneity and symmetric Jacobian conditions. However, DRUNet has a spectral radius greater than $1$. Yet, to ensure a strongly passive deep neural network-based denoiser, various strategies could be been envisioned, such as spectral normalization \cite{ryu2019plug}.
   
\renewcommand*{\arraystretch}{1.6}
\setlength{\tabcolsep}{4pt}
\begin{table}
\caption{Numerical experiments to assess the  RED conditions.}\label{tab:RED_conditions}
\centering
\begin{tabular}{lccc} 
\hline 
 &  & \textbf{FFHQ} & \textbf{ImageNet}  \\ 
\hline
 \multirow{2}{*}{\textbf{Homogeneity} }   &  $\mathrm{NMSE}_{{\mathrm{LH},1}}$  & $1.78\times 10^{-8}$  & $6.23 \times 10^{-8}$    \\
\cdashline{2-4} 
 &$\mathrm{NMSE}_{{\mathrm{LH},2}}$ & $1.75\times 10^{-2}$ & $6.21 \times 10^{-2}$   \\
\hline
 \textbf{Jacobian symmetry}  & $\mathrm{NMSE}_{{\mathrm{JS}}}$ & $5.81 \times 10^{-2}$ & $7.56\times 10^{-2}$  \\
\hline
\textbf{Strong passivity}  & ${\mathrm{MSR}} $  & $1.44$ & $1.76$  \\
\hline
\end{tabular}
\end{table}

\subsection{RED-ULA vs.~PnP-ULA} \label{subsec:PnPvsRED}

Section \ref{subsec:discussion_PnP} has drawn some  connections between PnP-ULA and RED-ULA. In particular, it has shown that in practice the two algorithms basically reduce to the same algorithmic scheme, except that PnP-ULA embeds an additional projection step onto an arbitrary pre-defined set $\mathbb{S}$. This projection aims at ensuring that the drift satisfies an  asymptotic growth condition. To experimentally validate this equivalence, Table~\ref{tab:comp_pnp_and_ula} reports the performance of PnP-ULA and RED-ULA when tackling the deblurring task. These performances have been computed over 100 images of the FFHQ dataset. Two configurations for the set $\mathbb{S}$ are considered, $\mathbb{S}=[0,1]$ and $\mathbb{S}=[-1,2]$. The rates of activation of the constraint, i.e., the proportion of samples generated by PnP-ULA that do not satisfy the drift condition and should be projected onto the set $\mathbb{S}$, are also reported in terms of percentage.
These results show that, when $\mathbb{S}=[-1,2]$,  the projection embedded in PnP-ULA is never activated and the performance is the same as the one obtained by RED-ULA, which confirms that the two algorithms are identical. When $\mathbb{S}=[0,1]$, this projection is applied to almost all the samples generated by PnP-ULA, without significantly affecting the performance. 

\setlength{\tabcolsep}{4pt}
\begin{table}
    \centering
    \caption{Numerical comparaisons of RED-ULA and PnP-ULA.}\label{tab:comp_pnp_and_ula}    
    \begin{tabular}{lcccccc}
    \hline
                     &   & \textbf{PSNR} & \textbf{SSIM} & \textbf{LPIPS} & \textbf{IAT} & $\%$ \textbf{Proj.}\\
    \hline
    \multicolumn{2}{c}{\textbf{RED-ULA}}  & $38.03$ & $ 0.9809$ & $0.0053$ & $75.22$ & - \\
    \hline
    \multirow{2}{*}{\textbf{PnP-ULA}} & $\mathbb{S}=[0,1]$ & $38.46$ & $0.9813$ & $0.0048$ & $75.21$ & $96.71$ \\
     \cdashline{2-7}
      &  $\mathbb{S}=[-1,2]$ & $38.03$ & $0.9809$ & $0.0053$ & $75.22$ & $0$  \\
    \hline
\end{tabular}
\end{table}

\section{Conclusion}
This work built a Bayesian counterpart of the regularization-by-denoising (RED) engine, offering a data-driven framework to define prior distributions in Bayesian inversion tasks. It defined a new probability distribution from the RED potential, which was subsequently embedded into a Bayesian model as a prior distribution. Since the resulting RED posterior distribution was not standard, a dedicated Monte Carlo algorithm was designed. By leveraging an asymptotically exact data augmentation (AXDA), this algorithm was a particular instance of the split Gibbs sampler which had the great advantage of decoupling the data-fitting term and the RED potential. One stage of SGS was performed following a Langevin Monte Carlo step, which leads to the so-called Langevin-within-split Gibbs sampling. A thorough theoretical analysis was conducted to assess the convergence guarantees of the algorithm. Some tight connections were drawn between AXDA and RED to show that the implicit prior resulting from an AXDA scheme coincides with the RED prior defined by a MMSE denoiser. Extensive numerical experiments showed that the proposed approach competes favorably with state-of-the-art variational and Monte Carlo methods when tackling conventional inversion tasks, namely deblurring, inpainting and super-resolution. The proposed approach was shown to provide a comprehensive characterization of the solutions which could be accompanied by an uncertainty quantification. By bridging the gap between the RED paradigm and Bayesian inference, this work opened new avenues for incorporating data-driven regularizations into Monte Carlo algorithms.

\appendices

\section{Proof of Proposition~\ref{prop1}}\label{proof:prop1}
    Under the RED conditions, the RED prior \eqref{prior} can be rewritten according to the pseudo-quadratic form
    \begin{equation*}
        p_{\textrm{red}}(\bx) \propto \exp\left[-\tfrac{\beta}{2} \bx^\top \boldsymbol{\Lambda}(\bx) \bx\right]
    \end{equation*}
    with $\boldsymbol{\Lambda}(\bx) = \mathbf{I}_n - \nabla \mathsf{D}_{\nu}(\bx)$. From Assumption~\ref{ass:well_defined_density}, there exists  $\lambda_{\text{min}} > 0$ such that $\lambda_{\text{min}} \mathbf{I}_{n} \preceq \boldsymbol{\Lambda}(\bx)$, $\forall \bx \in \mathbb{R}^n$. This implies that  $\lambda_{\text{min}} \bx^\top  \bx \le \bx^\top \boldsymbol{\Lambda}(\bx) \bx $ and
    \begin{equation*}
         \int_{\mathbb{R}^n} p_{\textrm{red}}(\bx) \dd\bx \le \int_{\mathbb{R}^n} \exp \left[ - \tfrac{\beta}{2} \lambda_{\text{min}} \|\bx \|^2  \right] \dd\bx < \infty.
    \end{equation*}

\section{Proof of Proposition~\ref{prop:convergence}}\label{proof_conv}
The proposed algorithm LwSGS does not fall into the class of Metropolis-within-Gibbs samplers. Hence, the first step of the analysis consists in demonstrating that  the Markov kernel $P_{\rho,\gamma}$ associated with the homogeneous Markov chain produced by LwSGS possesses a unique invariant distribution and is geometrically ergodic. This is achieved through an appropriate synchronous coupling, reducing the convergence analysis of $(\bV^{(t)})_{t \in \mathbb{N}} = (\bX^{(t)}, \bZ^{(t)})_{t \in \mathbb{N}}$ to that of the marginal process $(\bZ^{(t)})_{t \in \mathbb{N}}$.

The Markov chain $(\bV^{(t)})_{t \in \mathbb{N}}$ is associated with the Markov kernel defined, for any $\bv = (\bx,\bz) \in \mathbb{R}^{n} \times \mathbb{R}^{n}$ and $\mathsf{A} \in \mathcal{B}(\mathbb{R}^{n})$, $\mathsf{B} \in \mathcal{B}( \mathbb{R}^n)$, by
\begin{equation}\label{eq:def:prop2:P_rho_gamma}
  P_{\rho, \gamma}(\bv,\mathsf{A} \times \mathsf{B}) = \int_{\mathsf{B}} Q_{\rho,\gamma}(\bz,\dd {\tilde{\mathbf{z}}}|\bx)\int_{\mathsf{A}} p(\dd\bx|\by,\tilde{\mathbf{z}}) ,
\end{equation}
where $p(\cdot|\by, \tilde{\mathbf{z}}) \propto \exp \left[ -  f(\cdot,\by) - \tfrac{1}{2\rho^2}|| \cdot- \tilde{\bz}||^2 \right]$ 
and $Q_{\rho, \gamma}(\bz,\mathsf{B}|\bx)$ is the conditional Markov transition kernel given by 
\begin{equation*}\label{eq:Q_rho_gamma}
     Q_{\rho, \gamma}(\bz,\mathsf{B}|\bx) = \int_{\mathsf{B}} \exp \left[ -\frac{1}{4\gamma}\Bigr\| \tilde{\bz} - b_{\rho, \gamma}(\bx, \bz) \Bigr\|^2 \right] \frac{\dd\tilde{\bz}}{(4\pi \gamma)^{n/2}}
\end{equation*}
with $b_{\rho, \gamma}(\bx, \bz)=\left(1-\frac{\gamma}{\rho^2} \right)\bz  - \frac{\gamma}{\rho^2} \bx + \gamma \grad g(\bz)$.

Consider two independent sequences $(\boldsymbol{\xi}^{(t)})_{t \ge 1}$  and $(\boldsymbol{\eta}^{(t)})_{t\ge 1}$ of independent and identically distributed (i.i.d.) $n$-dimensional standard Gaussian random variables. Let introduce the stochastic processes $(\bV^{(t)},\tilde{\bV}^{(t)})_{t\ge 0}$ starting from $(\bv^{(0)},\tilde{\bv}^{(0)}) = ((\bx,\bz)^{\top},(\tilde{\bx},\tilde{\bz})^{\top})$ and recursively defined as ($t \ge 0$)
\begin{align}\label{eq:def:X}
      \bV^{(t+1)} &= \left(\bX^{(t+1)}, \bZ^{(t+1)} \right)^{\top}\\
    \tilde{\bV}^{(t+1)} &= \left(\tilde{\bX}^{(t+1)}, \tilde{\bZ}^{(t+1)} \right)^{\top}
\end{align}
with
\begin{align}
      \bZ^{(t+1)} &= (1-\frac{\gamma}{\rho^2}) \bZ^{(t)} + \frac{\gamma}{\rho^2} \bX^{(t)} - \gamma \beta \grad g(\bZ^{(t)}) + \sqrt{2\gamma} \boldsymbol{\eta}^{(t+1)} \nonumber \\
      \tilde{\bZ}^{(t+1)} & = (1-\frac{\gamma}{\rho^2}) \tilde{\bZ}^{(t)} + \frac{\gamma}{\rho^2} \tilde{\bX}^{(t)} - \gamma \beta \grad g(\tilde{\bZ}^{(t)}) + \sqrt{2\gamma} \boldsymbol{\eta}^{(t+1)}
      \label{eq:coupling_Z}
\end{align}
and 
\begin{equation}\label{eq:coupling_theta}
   \begin{array}{r c l}
     \bX^{(t+1)} = \boldsymbol{\mu}(\bZ^{(t+1)}) + \bQ^{-\tfrac{1}{2}} \boldsymbol{\xi}^{(t+1)},\\
      \tilde{\bX}^{(t+1)} = \boldsymbol{\mu}(\tilde{\bZ}^{(t+1)}) + \bQ^{-\tfrac{1}{2}} \boldsymbol{\xi}^{(t+1)},
   \end{array}
\end{equation}
where $\boldsymbol{\mu}(\cdot)$ and $\bQ$ are defined in \eqref{eq:def_mu_and_Q}. Note that $\bV^{(t)}$ and $\tilde{\bV}^{(t)}$ are distributed according to $\delta_{\bv}P_{\rho,\gamma}^{t}$ and $\delta_{\tilde{\bv}}P_{\rho,\gamma}^{t}$, respectively.
Hence, by definition of the Wasserstein distance of order 2, it follows that 
\begin{equation}\label{eq:W2X_def_1}
W_2(\delta_{\bv}P_{\rho,\gamma}^{t},\delta_{\tilde{\bv}}P_{\rho,\gamma}^{t}) \le \mathbb{E}\br{\|\bV^{(t)}-\tilde{\bV}^{(t)}\|^{2}}^{\nicefrac{1}{2}}.
\end{equation}
In this section, our attention is directed towards establishing an upper limit for the squared norm $\|\bV^{(t)}-\tilde{\bV}^{(t)}\|^2$. This upper bound yields a specific limit on the Wasserstein distance, utilizing the preceding inequality. We begin by constraining \eqref{eq:W2X_def_1} through the utilization of the subsequent technical lemma.

\begin{lemma}\label{lem:W2_dirac}
Let $\gamma \in \mathbb{R}_+^*$. If Assumptions~\ref{ass:twice_differentiable}-\ref{ass:supp_fort_convex} hold then, for any $\bv=(\bx,\bz)^{\top}$ and $\tilde{\bv}=(\tilde{\bx},\tilde{\bz})^{\top}$, for any $t\ge 1$, 
\begin{equation*}
W_2(\delta_{\bv} P_{\rho, \gamma}^{t}, \delta_{\tilde{\bv}} P_{\rho, \gamma}^{t})
\le \kappa_{\gamma}^{t-1} C_1^{\nicefrac{1}{2}}\br{\kappa_{\gamma}\|\bz-\tilde{\bz}\| + \frac{\gamma}{\rho^2} \|\bx-\tilde{\bx}\|}
\end{equation*}
where $\kappa_{\gamma} = \max \ac{|1 - \gamma \beta m_g|, |1 - \gamma (\beta M_g + 1/\rho^2)|}$ and $C_1=1 + \frac{1}{\rho^2} \| \bQ^{-1}\|^2$ .
\end{lemma}
\begin{proof}
  Consider $(\bV^{(t)},\tilde{\bV}^{(t)})_{t \in\mathbb{N}}$ defined in \eqref{eq:def:X}. Let $t \ge 0$. By \eqref{eq:coupling_theta}, we have $\bX^{(t+1)}-\tilde{\bX}^{(t+1)} = \frac{1}{\rho^2} \bQ^{-1} (\bZ^{(t+1)} - \tilde{\bZ}^{(t+1)})$
which implies 
\begin{align}\label{eq:geo_V}
\|\bV^{(t+1)}&- \tilde{\bV}^{(t+1)} \|^2\\
&\le \left(1+ \frac{1}{\rho^2}\| \bQ^{-1}\|^2 \right) \|\bZ^{(t+1)} - \tilde{\bZ}^{(t+1)} \|^2 \nonumber . 
\end{align}
By \eqref{eq:W2X_def_1} and \ref{eq:geo_V}, we need to bound $\| \bZ^{(t)}-\tilde{\bZ}^{(t)}\|^2$. 
By \eqref{eq:coupling_Z}, we have
\begin{multline*}
\label{eq:zi_recursion}
  \bZ^{(t+1)} - \tilde{\bZ}^{(t+1)} = (1-\frac{\gamma}{\rho^2}) (Z^{(t)} - \tilde{\bZ}^{(t)}) + \frac{\gamma}{\rho^2}(\bX^{(t)}-\tilde{\bX}^{(t)}) \notag
 \\ 
 - \gamma \beta \pr{\grad g(\bZ^{(t)})-\grad g(\tilde{\bZ}^{(t)})}. 
\end{multline*}
Since $g(\cdot)$ is twice differentiable, we have
\begin{multline*}
    \grad g(\bZ^{(t)})-\grad g(\tilde{\bZ}^{(t)}) = \int_0^1 \grad^2 g(\tilde{\bZ}^{(t)} + t(\bZ^{(t)}-\tilde{Z}^{(t)}))\,\dd t \\
    \times (\bZ^{(t)}-\tilde{\bZ}^{(t)}).
\end{multline*}
Using $\bX^{(t)}-\tilde{\bX}^{(t)}  = \frac{1}{\rho^2} \bQ^{-1} (\bZ^{(t)} - \tilde{\bZ}^{(t)})$, it follows that
\begin{multline*}
  \bZ^{(t+1)}-\tilde{\bZ}^{(t+1)} = \pr{ \Big[1-\frac{\gamma}{\rho^2}\Big] \mathbf{I}_{n} + \frac{1}{\rho^4} \bQ^{-1}} (\bZ^{(t)}-\tilde{\bZ}^{(t)}) \\
- \gamma \beta \int_0^1 \grad^2 g(\tilde{\bZ}^{(t)} + t(\bZ^{(t)} -\tilde{\bZ}^{(t)}))\,\dd t \cdot(Z^{(t)} -\tilde{\bZ}^{(t)}). 
\end{multline*}
Let us define $\tilde{\mathbf{A}}^{\top} = \left[\frac{1}{\sigma} \mathbf{A}^{\top} \quad  \frac{1}{\rho}\mathbf{I}_n\right]$. The precision matrix $\bQ$ can then be rewritten as $\bQ = \tilde{\mathbf{A}}^{\top}\tilde{\mathbf{A}} $. We define the orthogonal projector 
\begin{equation}\label{eq:def_projection}
\mathbf{P} = \tilde{\mathbf{A}} \bQ^{-1} \tilde{\mathbf{A}}^{\top}.  
\end{equation}
By denoting $\mathbf{J}=[\mathbf{0}_n \quad \frac{1}{\rho}\mathbf{I}_n]^\top$, we have
$\mathbf{J}^\top \mathbf{P} \mathbf{J} = \frac{1}{\rho^4} \bQ^{-1}
$. Considering 
$$\mathbf{D}_g^{(t)} = \frac{\gamma}{\rho^2} \mathbf{I}_n +  \gamma \beta \int_0^1 \grad^2 g\left(\tilde{\bZ}^{(t)} + s(\bZ^{(t)} -\tilde{\bZ}^{(t)} )\right) \dd s$$
and the projection matrix $\mathbf{P}$ defined in \eqref{eq:def_projection}, the difference $\bZ^{(t+1)}-\tilde{\bZ}^{(t+1)}$ can be rewritten as
\begin{align*}
  \bZ^{(t+1)}-\tilde{\bZ}^{(t+1)} &= \pr{ \mathbf{I}_n - \mathbf{D}_g^{(t)} + \gamma \mathbf{J}^\top  \mathbf{P} \mathbf{J}} (\bZ^{(t)}-\tilde{\bZ}^{(t)}).
\end{align*}
Under Assumptions~\ref{ass:twice_differentiable}-\ref{ass:supp_fort_convex} and exploiting the fact that $\mathbf{P}$ is an orthogonal projector, thus $  \mathbf{P} \preccurlyeq \mathbf{I}$, we have
\begin{equation*}
 \Big[1- \gamma (\beta M_g + \frac{1}{\rho^2})\Big] \mathbf{I}_{n}
\preccurlyeq
\mathbf{I}_n - \mathbf{D}_g^{(t)} + \gamma \mathbf{J}^\top  \mathbf{P} \mathbf{J}
 \preccurlyeq
 \Big[1-  \gamma \beta m_g \Big] \mathbf{I}_{n} .
\end{equation*}
Therefore, we get $\| \bZ^{(t+1)}-\tilde{\bZ}^{(t+1)}\|  \le \kappa_{\gamma} \|\bZ^{(t)}-\tilde{\bZ}^{(t)}\|$. An immediate induction shows that, for all $t \ge 1$,
\begin{equation}
  \label{eq:contract_n_2}
  \|\bZ^{(t)}-\tilde{\bZ}^{(t)}\| \le \kappa_{\gamma}^{t-1} \|\bZ_1-\tilde{\bZ}_1\|.
\end{equation}
In addition, we have 
$\bZ_{1}-\tilde{\bZ}_{1} = (\mathbf{I}_n - \mathbf{D}_g^{(0)})(\bz-\tilde{\bz}) + \frac{\gamma}{\rho^2}(\bx-\tilde{\bx})$. The triangle inequality gives
\begin{align*}
\|\bZ_1-\tilde{\bZ}_1\|  &\le  \| \mathbf{I}_n - \frac{\gamma}{\rho^2} \mathbf{I}_n - \mathbf{D}_g^{(0)}\|\|\bz-\tilde{\bz}\| + \frac{\gamma}{\rho^2} \|\bx-\tilde{\bx} \| \\
& \le  \kappa_{\gamma}\|\bz-\tilde{\bz}\| + \frac{\gamma}{\rho^2}\|\bx-\tilde{\bx}\| .
\end{align*}
Combining \eqref{eq:contract_n_2} and the previous inequality and using \ref{eq:geo_V}, we get for $t \ge 1$,
\begin{multline*}
\| \bV^{(t)} - \tilde{\bV}^{(t)} \|^2
\le \kappa_{\gamma}^{2(t-1)} \Bigr( 1 + \|\frac{1}{\rho^2} \bQ^{-1} \|^2 \Bigr) \\
\times \Bigr( \kappa_{\gamma}\|\bz-\tilde{\bz}\| + \frac{\gamma}{\rho^2} \|\bx-\tilde{\bx}\| \Bigr)^2.
\end{multline*}
The proof is concluded by \eqref{eq:W2X_def_1}.
\end{proof}
On the basis of the preceding lemma, we provide in what follows the proof of Proposition~\ref{prop:convergence}.

\begin{proof}
Note that the condition $0 <\gamma \le 2(\beta m_g + \beta M_g + 1/\rho^2)^{-1}$ ensures that $\kappa_{\gamma} = 1 - \gamma \beta m_g \in(0,1)$. From Lemma~\ref{lem:W2_dirac} combined with \cite[Lemma 20.3.2, Theorem 20.3.4]{douc2018markov},  $P_{\rho, \delta}$ admits a unique invariant probability distribution $\pi_{\rho,\gamma}$.  Moreover, for any $\bv~=~ (\bx,\bz)^{\top}$ with $\bx \in \mathbb{R}^n$, $\bz \in \mathbb{R}^n$ and any $t \in \mathbb{N}^*$, we have
    \begin{align*}
      &W_{2}^2(\delta_{\bv} P_{\rho,\gamma}^{t}, \pi_{\rho,\gamma}) \le
      \kappa_{\gamma}^{2(t-1)} \Big(1 +  \|\frac{1}{\rho^2} \bQ^{-1} \|^2 \Big) \\
 &\times \int_{\mathbb{R}^{n}\times \mathbb{R}^{n}}\Big[\bigr(1-\gamma \beta m_g\bigr) \|\bz-\tilde{\bz}\| + \frac{\gamma}{\rho^2} \|\bx-\tilde{\bx}\|\Big]^2 \dd \pi_{\rho,\gamma}(\tilde{\bv}).
    \end{align*}
\end{proof}

\section{Proof of Proposition~\ref{prop:bias}}\label{proof_bias}
In this section, we establish an explicit bound on $W_2^2(\pi_{\rho,\gamma}, \pi_{\rho})$ where $\pi_{\rho}$ is the target augmented  distribution. Consider first for any $\bx \in \mathbb{R}^{n}$, the stochastic differential equation (SDE) defined by
\begin{equation}\label{eq:def:cont_process}
\dd \tilde{\bY}_t^{\bx} = - \nabla U(\tilde{\bY}^{ \bx}_t)\,\dd t - \tfrac{1}{\rho^2}  \bx + \sqrt{2}\,\dd \bB_t
\end{equation}
where $(\bB_t)_{t\ge 0}$ is a $n$-dimensional Brownian motion and the potential $U_{\bx}(\cdot)$ is defined as
\begin{equation}
	\label{eq:def:cont_Utilde}
	U_{\bx}(\cdot) = \beta g(\cdot) + \tfrac{1}{2\rho^2}\|\cdot- \bx\|^2
\end{equation}
and, to lighten the notations, we denote the potential in \ref{eq:def:cont_process}, $U(\cdot)= U_{\boldsymbol{0}}(\cdot)$.
Note that under Assumption~\ref{ass:twice_differentiable}, this SDE admits a unique strong solution \cite[Chapter IX, Theorem (2.1)]{revuz2013continuous}. Denote the Markov semi-group associated to \eqref{eq:def:cont_process} by $(\tilde{R}_{\rho,t})_{t\ge 0}$ defined for any $\tilde{\by}_0 \in \mathbb{R}^{n}$, $t\ge 0$ and $\mathsf{B} \in \mcb{\mathbb{R}^{n}}$ by
$\tilde{R}_{\rho,t}(\tilde{\by}_0,\mathsf{B}|\bx)~=~ \mathbb{P}(\tilde{\bY}^{ \bx,\tilde{\by}_0}_t\in \mathsf{B}),$
where $(\tilde{\bY}_t^{\bx,\tilde{\by}_0})_{t\ge 0}$ is a solution of \eqref{eq:def:cont_process} with $\tilde{\bY}_0^{\bx,\tilde{\by}_0}=\tilde{\by}_0$. We consider the Markov kernel defined, for any $\bv = (\bx,\bz)^{\top}$ and $\mathsf{A} \in \mathcal{B}(\mathbb{R}^{n})$, $\mathsf{B} \in \mathcal{B}( \mathbb{R}^n)$, by
\begin{equation}\label{eq:def:P_tilde}
  \tilde{P}_{\rho, \gamma}(\bv,\mathsf{A} \times \mathsf{B}) = \int_{\mathsf{B}} \tilde{R}_{\rho, \gamma}(\bz,\dd {\tilde{\mathbf{z}}}|\bx)\int_{\mathsf{A}}p(\dd \bx|\by, \tilde{\mathbf{z}})\,.
\end{equation}
Note that $P_{\rho, \gamma}$ can be interpreted as a discretized version of $\tilde{P}_{\rho, \gamma}$ using the Euler-Maruyama scheme.
Under Assumption~\ref{ass:twice_differentiable}, the Markov kernel $\tilde{P}_{\rho, \gamma}$ defined by \eqref{eq:def:P_tilde} admits $\pi_{\rho}$ as an invariant probability distribution \cite[Proposition S21]{plassier2021dg}, i.e, $\forall t \ge 0, \, \pi_{\rho}\tilde{P}_{\rho, \gamma}^{t} = \pi_{\rho}$. \\

Let $(\bB_{t})_{t\ge 0}$ an  i.i.d. $n$-dimensional Brownian motion and let $(\boldsymbol{\xi}^{(t)})_{t\ge 0}$ be a sequence of i.i.d. standard $n$-dimensional Gaussian random variables independent of $(\bB_t)_{t\ge 0}$. For ${t \ge 0}$, we define the synchronous coupling $\bV^{(t)}=(\bX^{(t)}, \bZ^{(t)})$ and $(\tilde{\bV}^{(t)} =(\tilde{\bX}^{(t)}, \tilde{\bZ}^{(t)})$, starting from $(\bX^{(0)}, \bZ^{(0)}) = (\bx, \bz)$, $(\tilde{\bX}^{(0)}, \tilde{\bZ}^{(0)})$ distributed according to $\pi_{\rho}$
\begin{align}
  \label{eq:cont_coupling_2}
  &\tilde{\bZ}^{(t+1)} = \tilde{\bY}_{(t+1) \gamma},  &\tilde{\bX}^{(t+1)} = \boldsymbol{\mu}(\tilde{\bZ}^{(t+1)} ) + \bQ^{-\tfrac{1}{2}} \boldsymbol{\xi}^{(t+1)}, \nonumber\\
  &\bZ^{(t+1)} = \bY_{(t+1)\gamma},                   &\bX^{(t+1)} = \boldsymbol{\mu}(\bZ^{(t+1)}) + \bQ^{-\tfrac{1}{2}} \boldsymbol{\xi}^{(t+1)},
\end{align}
with
\begin{align}
  \label{eq:cont_coupling_Y_N}
\tilde{\bY}_{(t+1)\gamma} = \tilde{\bY}_{t\gamma} &- \int_{t\gamma}^{(t+1)\gamma} \nabla U(\tilde{\bY}_{l})\,\dd l \\
                          & + \frac{\gamma}{\rho^{2}}  \tilde{\bX}^{(t)} + \sqrt{2}(\bB_{(t+1)\gamma} - \bB_{t\gamma}) \nonumber\\
\bY_{(t+1)\gamma}         = \bY_{t\gamma} &- \gamma \nabla U(\bY_{t\gamma}) \nonumber \\
                          &+ \frac{\gamma}{\rho^{2}}  \bX^{(t)} + \sqrt{2}(\bB_{(t+1)\gamma} - \bB_{t\gamma}). \nonumber
\end{align}
The stochastic processes $(\bV^{(t)}, \tilde{\bV}^{(t)})_{t \ge 0}$ satisfy \eqref{eq:geo_V}. Note that $\bV^{(t)}$ and $\tilde{\bV}^{(t)}$ are distributed according to $\pi_{\rho}\tilde{P}_{\rho, \gamma}^{t} = \pi_{\rho}$ and $\delta_{\tilde{\bv}}P_{\rho,\gamma}^{t}$, respectively. Hence, by definition of the Wasserstein distance of order 2, it follows that
\begin{equation}\label{eq:W2X_def}
W_2(\pi_{\rho},\delta_{\tilde{\bv}}P_{\rho,\gamma}^{t}) \le \mathbb{E}\br{\|\bV^{(t)}-\tilde{\bV}^{(t)}\|^{2}}^{\nicefrac{1}{2}}.
\end{equation}

We start this section by a first estimate on $\mathbb{E}\br{\|\bV^{(t)}-\tilde{\bV}^{(t)}\|^{2}}^{\nicefrac{1}{2}}$.
The following result holds regarding the process $(\tilde{\bY}_t)_{t \in \mathbb{N}_+}$ defined in \eqref{eq:cont_coupling_Y_N}. 

\begin{lemma}\label{lem:bound:cont_Z_nduction1}
If Assumptions~\ref{ass:twice_differentiable} and \ref{ass:supp_fort_convex} hold, let define $\tilde{M} = \beta M_g + 1/\rho^2$ and $\gamma \in \mathbb{R}_+^*$ such that $ \gamma < 1/\tilde{M}$.
Then, for any $t \ge 1$,
\begin{align}\label{eq:eq:T2_new_bound}
\tilde{\bZ}^{(t+1)}-\bZ^{(t+1)} = \mathbf{\mathbf{T}}_{1}^{(t)}(\tilde{\bZ}^{(t)} - \bZ^{(t)}) - \mathbf{\mathbf{T}}_{2}^{(t)},
\end{align}
where $(\bZ^{(t)}, \tilde{\bZ}^{(t)})_{t \in \mathbb{N}}$ is defined in \eqref{eq:cont_coupling_2}, $\mathbf{\mathbf{T}}_{1}^{(t)}$ and $\mathbf{\mathbf{T}}_{2}^{(t)}$ are given by 
\begin{align}\label{eq:def:T}
 &\mathbf{\mathbf{T}}_1^{(t)} = \mathbf{I}_{n} + \gamma \mathbf{J}^\top \mathbf{P} \mathbf{J} - \gamma \int_0^1 \nabla^2 U((1-s) \bY_{t\gamma} + s \tilde{\bY}_{t\gamma}) \dd s \\
\label{eq:def:T2}
& \mathbf{\mathbf{T}}_{2}^{(t)}
=  \int_0^{\gamma} \left[ \nabla U (\tilde{\bY}_{t\gamma + l}^{(t)}) - \nabla U (\tilde{\bY}_{t\gamma}^{(t)})\right] \dd l .
\end{align}
\end{lemma}
\begin{proof}
Since $U(\cdot)$ is twice differentiable, we have
\begin{multline*}
 \nabla U(\tilde{\bY}_{t\gamma}) - \nabla U(\bY_{t\gamma}) \\
= \Big(\int_0^1 \nabla^2 U((1-s) \bY_{t\gamma} + s \tilde{\bY}_{t\gamma})\,\dd s \Big) 
\times (\tilde{\bY}_{t\gamma} - \bY_{t\gamma}).
\end{multline*}
By using $ \frac{1}{\rho^2} (\tilde{\bX}^{(t)} - \bX^{(t)}) = \mathbf{J}^\top \mathbf{P}\mathbf{J} (\tilde{\bY}_{t\gamma} - \mathrm{\bY}_{t\gamma})$, it follows from \eqref{eq:cont_coupling_Y_N} that
\begin{align*}\label{eq:rec_Y_proof_cont_1}
 &\tilde{\bY}_{(t+1)\gamma} - \bY_{(t+1)\gamma} = (\tilde{\bY}_{t\gamma} - \mathrm{\bY}_{t\gamma})\\
 & \times \Big(\mathbf{I}_{n} - \gamma \int_0^1 \nabla^2 U((1-s) \bY_{t\gamma} + s \tilde{\bY}_{t\gamma}) \dd s 
+ \gamma \mathbf{J}^\top \mathbf{P}\mathbf{J} \Big) \\
&  - \int_0^{\gamma}\big[\nabla U (\tilde{\bY}_{t\gamma+l}) - \nabla U (\tilde{\bY}_{t\gamma})\big]\,\dd l.
\end{align*}
where $\mathbf{P}$ is defined in \eqref{eq:def_projection}. By \eqref{eq:cont_coupling_2}, we have 
\begin{align*}
\tilde{\bZ}^{(t+1)}-\bZ^{(t+1)} = \mathbf{\mathbf{T}}_{1}^{(t)}(\tilde{\bZ}^{(t)} - \bZ^{(t)}) - \mathbf{\mathbf{T}}_{2}^{(t)}.
\end{align*}
\end{proof}

Based on Lemma~\ref{lem:bound:cont_Z_nduction1}, we have the following relation between $\| \tilde{\bZ}^{(t+1)} - \bZ^{(t+1)} \|^2$ and $\| \tilde{\bZ}^{(t)} - \bZ^{(t)} \|^2$.

\begin{lemma}\label{lem:bound:cont_Z_induction}
If Assumption~\ref{ass:supp_fort_convex} holds, let define $\tilde{M} = \beta M_g + 1/\rho^2$ and $\gamma \in \mathbb{R}_+^*$ such that $\gamma < 1/\tilde{M}$. Then, for any $\epsilon >0$ and $t \ge 1$,
\begin{align*}
 \| \tilde{\bZ}^{(t+1)} &- \bZ^{(t+1)} \|^2
\\ &\le (1+2\epsilon) \normn{\mathbf{\mathbf{T}}_1^{(t)}}^2 \| \tilde{\bZ}^{(t)} - \bZ^{(t)} \|^2 
+ (1+\tfrac{1}{2\epsilon}) \normn{\mathbf{\mathbf{T}}_{2}^{(t)}}^2.
\end{align*}
where $(\bZ^{(t)}, \tilde{\bZ}^{(t)})_{t \in \mathbb{N}}$ is defined in \eqref{eq:cont_coupling_2}.
\end{lemma}
\begin{proof}
  The proof of Lemma~\eqref{lem:bound:cont_Z_induction} can be found in \cite{plassier2021dg}.
\end{proof}

We now have the following result regarding the contracting term.

\begin{lemma}
\label{lem:bound:cont_T1_bis}
Let define $\tilde{M} =\beta M_g + 1/\rho^2$ and $\gamma \in \mathbb{R}_+^*$ such that $\gamma<1/\tilde{M}$. If assumptions~\ref{ass:twice_differentiable} and \ref{ass:supp_fort_convex} hold then, for any $n\ge0$, 
\begin{align*}
\normn{\mathbf{\mathbf{T}}_1^{(t)}} &\le 1 - \gamma \beta m_g,
\end{align*}
where $\mathbf{\mathbf{T}}_1^{(t)}$ is defined in \eqref{eq:def:T}.
\end{lemma}
\begin{proof}
By denoting
\begin{align*}
    \mathbf{C}^{(t)} &=\gamma \int_0^1 \nabla^2 U((1-s) \bY_{t\gamma} + s \tilde{\bY}_{t\gamma})\,\dd s \\
    &=\frac{\gamma}{\rho^2} \bI_n + \gamma \int_0^1 \nabla^2 g((1-s) \bY_{t\gamma} + s \tilde{\bY}_{t\gamma})\,\dd s
\end{align*}
Assumptions~\ref{ass:twice_differentiable} and \ref{ass:supp_fort_convex} imply 
$\bigr( \dfrac{\gamma}{\rho^2} + \gamma \beta m_g \bigr) \mathbf{I}_{n}
\preccurlyeq \mathbf{C}^{(t)}
\preccurlyeq \bigr( \dfrac{\gamma}{\rho^2} + 2 \gamma \beta \bigr) \mathbf{I}_{n}.$
Since $\mathbf{P}$ is an orthogonal projection \eqref{eq:def_projection},  $\mathbf{P} \preccurlyeq \mathbf{I}$ and 
$
\mathbf{0}_{n}
\preccurlyeq  \gamma \mathbf{J}^\top \mathbf{P} \mathbf{J}
\preccurlyeq  \dfrac{\gamma}{\rho^2} \mathbf{I}_{n}$.
Subtracting these previous inequalities and adding $\mathbf{I}_n$ leads to
\begin{equation*}
 \bigr( 1 -\gamma \tilde{M} \bigr) \mathbf{I}_{n}
\preccurlyeq \mathbf{I}_n - \mathbf{C}^{(t)} + \gamma \mathbf{J}^\top \mathbf{P} \mathbf{J}
\preccurlyeq \bigr( 1 - \gamma \beta m_g \bigr) \mathbf{I}_{n}   
\end{equation*}
and
\begin{equation*}
  \normn{\mathbf{\mathbf{T}}_1^{(t)}} \le \max (|1-\gamma \beta m_g|, |1-\gamma\tilde{M}|) = 1 - \gamma \beta m_g.
\end{equation*}
\end{proof}

The following lemma provides an upper bound on $\normn{\mathbf{\mathbf{T}}_2^{(t)}}^2$.

\begin{lemma}\label{lem:bound:cont_expec_T2}
  If Assumptions~\ref{ass:twice_differentiable} and \ref{ass:supp_fort_convex} hold, let define $\tilde{m} = \beta m_g + 1/\rho^2$, $\tilde{M} = \beta M_g + 1/\rho^2$ and $\gamma \in \mathbb{R}_+^*$ such that $\gamma < 1/\tilde{M}$. Then, for any $t \in \mathbb{N}$, 
\begin{align*}{\mathbb{E}\br{\normn{\mathbf{\mathbf{T}}_2^{(t)}}^2}}
&\le n \gamma^2\tilde{M}^2 \br{1 + \frac{\gamma^2\tilde{M}^2}{12}+ \frac{\gamma \tilde{M}^2}{2\tilde{m}}},
\end{align*}
where $\mathbf{T}_{2}^{(t)}$ is defined in \eqref{eq:def:T2}.
\end{lemma}
\begin{proof}
Let $t \in \mathbb{N}$. We have
\begin{align}
\normn{\mathbf{\mathbf{T}}_2^{(t)}}^2
&= \normBig{\int_0^{\gamma} \big[\nabla U (\tilde{\bY}_{t\gamma+l}) - \nabla U (\tilde{\bY}_{t\gamma})\big]\,\dd l}^2 \label{eq:integral_bound_cont}.
\end{align}
With the previous result and the Jensen inequality, 
\begin{equation}\label{eq:bound:cont_T2}
\normn{\mathbf{\mathbf{T}}_2^{(t)}}^2
\le \int_0^{\gamma} \normBig{\nabla U (\tilde{\bY}_{t\gamma+l}) - \nabla U (\tilde{\bY}_{t\gamma})}^2\,\dd l.
\end{equation}
Let denote $\mathcal{G}_0 = \sigma(\bZ^{(0)},\tilde{\bZ}^{(0)},\bX^{(0)},\tilde{\bX}^{(0)})$, for any $t \in \mathbb{N}^{*}$,  
$\mathcal{G}_{t} = \sigma\{(\bZ^{(0)},\tilde{\bZ}^{(0)},\bX^{(0)},\tilde{\bX}^{(0)}), (\bB_k)_{k\ge 0}, k \le t\}$
and 
$\mathcal{F}_{t} \text{ the } \sigma\text{-field generated by } \mathcal{G}_{t-1}$.
Using \cite[Lemma 21]{durmus2019high} applied to the potential \eqref{eq:def:cont_Utilde} 
yields
\begin{align}\nonumber
&\int_0^{\gamma} \mathbb{E}^{\mathcal{F}_{t\gamma}}\normn{\nabla U (\tilde{\bY}_{t\gamma+l}) - \nabla U (\tilde{\bY}_{t\gamma})}^2\,\dd l \\
&=\nonumber \int_0^{\gamma} \mathbb{E}^{\mathcal{F}_{t\gamma}}\normn{\nabla U_{\tilde{\bx}^{(t)}} (\tilde{\bY}_{t\gamma+l}) - \nabla U_{\tilde{\bx}^{(t)}} (\tilde{\bY}_{t\gamma})}^2\,\dd l \\
\label{eq:bound:cont_{i}nt}
&\le \gamma^2\tilde{M}^2\br{n + \frac{n \gamma^2\tilde{M}^2}{12}+ \frac{\gamma\tilde{M}^2}{2} \| \tilde{\bY}_{t\gamma} - \bz_{t,\star} \|^2},
\end{align}
where $\bz_{t,\star}= \argmin_{\bz \in \mathbb{R}^{n}} U_{\tilde{\bx}^{(t)}}(\bz)$.

By \eqref{eq:bound:cont_{i}nt} and using \cite[Proposition 1]{durmus2019high}, we get
\begin{multline*}
\int_0^{\gamma} \mathbb{E}\| \nabla U (\tilde{\bY}_{t\gamma+l}) - \nabla U (\tilde{\bY}_{t\gamma})\|^2\,\dd l \le n \gamma ^2\tilde{M}^2 \\
\times [1 + \frac{\gamma^2\tilde{M}^2}{12}
+ \frac{\gamma \tilde{M}^2}{2\tilde{m}}] .
\end{multline*}
Combining this result with  \eqref{eq:bound:cont_T2} completes the proof.
\end{proof}

We can now combine Lemma~\ref{lem:bound:cont_expec_T2} and Lemma~\ref{lem:bound:cont_T1_bis} with Lemma~\ref{lem:bound:cont_Z_nduction1} to get the following bound.

\begin{lemma}\label{lem:bound:cont_expec_norm_spe}
If Assumptions~\ref{ass:twice_differentiable} and \ref{ass:supp_fort_convex} hold, let denote $\tilde{m} =\beta m_g + 1/\rho^2$, $\tilde{M} = \beta M_g + 1/\rho^2$,  $\gamma \in \mathbb{R}_+^*$ such that $\gamma <1/\tilde{M}$ and $r_{\gamma} = \gamma \beta m_g \in (0,1)$.
Then, for $t \ge 1$, we have
\begin{eqnarray*}
\mathbb{E}&\br{\normn{\tilde{\bZ}^{(t)} - \bZ^{(t)}}^2}
\le (1 - r_{\gamma} + r_{\gamma}^2/2)^{2(t-1)} \mathbb{E}\br{\normn{\tilde{\bZ}^{(1)} - \bZ^{(1)}}^2} \\
&+ 2(\beta m_g)^{-1} \times n \gamma \tilde{M}^2 \Bigr(1 + \frac{\gamma \tilde{M}^2}{12}+ \frac{\gamma \tilde{M}^2}{2\tilde{m}} \Bigr).
\end{eqnarray*}
\end{lemma}
\begin{proof}
The proof of Lemma~\eqref{lem:bound:cont_expec_norm_spe} can be found in \cite{plassier2021dg}.
\end{proof}

\begin{lemma}\label{lem:bound:cont_wass}
If Assumption~\ref{ass:supp_fort_convex} holds, let denote $\tilde{m} = \beta m_g + 1/\rho^2$, $\tilde{M} = \beta M_g + 1/\rho^2$, $\gamma\in \mathbb{R}_+^*$ such that $\gamma<1/\tilde{M}$ and $r_{\gamma} \in (0,1)$. Then, for any $\bv \in \mathbb{R}^{n \times n}$ and $t \ge 1$, the following holds
\begin{align*}
W_2^{2} (\delta_{\bv} P_{\rho, \gamma}^{t},\pi_{\rho})
&\le (1- r_{\gamma}+ r_{\gamma}^2/2)^{2(t-1)} (1 + \frac{1}{\rho^2} \|\bQ^{-1}\|^2)\\
& \times \mathbb{E}\br{ \| \tilde{\bZ}^{(1)} - \bZ^{(1)} \|^2} + \frac{2(1 + \frac{1}{\rho^2} \|\bQ^{-1}\|^2)}{\beta m_g} \\
&\times n \gamma \tilde{M}^2 [1 + \frac{\gamma^2\tilde{M}^2}{12}+ \gamma \tilde{M}^2/(2\tilde{m})],
\end{align*}
where $P_{\rho, \gamma}$ is defined in \eqref{eq:def:prop2:P_rho_gamma} and $(\tilde{\bZ}^{(t)},{\bZ}^{(t)})_{t \in \mathbb{N}}$ is defined in \eqref{eq:cont_coupling_2}.
\end{lemma}
\begin{proof}
The proof follows from the combination of \eqref{eq:W2X_def} with \eqref{eq:geo_V} and Lemma~\ref{lem:bound:cont_expec_norm_spe}.
\end{proof}

Now we can give the proof of Proposition~\ref{prop:bias}.
\begin{proof}
Note that the condition $\gamma \le 2(\beta m_g + \beta M_g +1/\rho^2)^{-1}$ ensures that 
$1-r_{\gamma}+ r_{\gamma}^2/2  \in (0,1) $. By Proposition~\ref{prop:convergence}, $\delta_{\bv} P_{\rho, \gamma}^{t}$ converges in $W_2$ to $\pi_{\rho, \gamma}$. Therefore, using Lemma~~\ref{lem:bound:cont_wass} and taking $t \rightarrow \infty$, we obtain
\begin{align*}
W_2^{2} (\pi_{\rho}, \pi_{\rho,\gamma})
& \le \frac{2(1 + \frac{1}{\rho^2} \| \bQ^{-1}\|^2) }{\beta m_g} \\
&\times n\gamma \tilde{M}^2 \Big(1 + \frac{\gamma^2\tilde{M}^2}{12}+ \frac{\gamma \tilde{M}^2}{2\tilde{m}}\Big) .
\end{align*}
\end{proof}


\section{Augmented distribution and RED-LwSGS for super-resolution}\label{sec:superresolution}
This appendix details the AXDA model and the corresponding sampling algorithm when tackling the super-resolution task. In this case, the operator $\bA$ can be written as
$\bA = \bS \bB$ where $\bB \in \mathbb{R}^{n \times n}$ is a circulant matrix standing for a spatially invariant blur and $\bS \in \mathbb{R}^{m \times n}$ stands for a regular downsampling operator.  When directly adopting the splitting trick proposed in Section \ref{subsec:LwSGS}, sampling according to the conditional distribution \eqref{eq:condx} remains difficult because the precision matrix is neither diagonal (as for the inpainting task) nor diagonalizable in the Fourier domain (as for the deblurring task). To overcome this difficulty, one suitable AXDA consists in introducing two splitting variables, which allows the operators $\bB$ and $\bS$ to be decoupled. This leads to the augmented posterior distribution
\begin{multline}
    \pi_{\rho_1, \rho_2}(\bx, \bz_1, \bz_2) \propto \exp \left[ - \dfrac{1}{2\sigma^2}||\bS \bz_1  - \by||_2^2 - \beta g_{\textrm{red}}(\bz_2) \right. \notag  \\
    \left. -\dfrac{1}{2\rho_1^2}||\bB \bx - \bz_1||_2^2  - \dfrac{1}{2\rho_2^2}|| \bx - \bz_2||^2 \right]. 
\end{multline}
The associated SGS alternatively samples according to the three conditional distributions
\begin{align}
    p(\bz_1 | \bx,\by)      &\propto \exp \left[ - \tfrac{1}{2\sigma^2}\left\|\textbf{S} \bz_1  - \by\right\|_2^2 - \tfrac{1}{2\rho_1^2}||{\bB} \bx - \bz_1||_2^2 \right] \label{eq:sr_p1}  \\
    p(\bx|\bz_1, \bz_2) &\propto \exp \left[ - \tfrac{1}{2\rho_1^2}\left\|{\bB} \bx - \bz_1\right\|_2^2 - \tfrac{1}{2\rho_2^2}\left\| \bx - \bz_2\right\|^2 \right] \label{eq:sr_p2} \\
    p(\bz_2 | \bx)      &\propto \exp \left[ -  \tfrac{1}{2}\bx^\top  \left( \bz_2 - \mathsf{D}_{\nu}(\bz_2) \right) - \tfrac{1}{2\rho_2^2}\left\| \bx - \bz_2\right\|^2 \right] \label{eq:sr_p3} \nonumber
\end{align}
It appears that \eqref{eq:sr_p1} and \eqref{eq:sr_p2} define the conditional posteriors associated with the inpainting and deblurring tasks, respectively.

\section{Experimental Details}\label{app:implementation}

\subsection{Pretrained denoiser}\label{app:subsec:denoiser}
All experiments have been performed with DRUNet as the   pre-trained denoiser used by the PnP- and RED-based methods. This denoiser $\mathsf{D}_{\nu}(\cdot)$ has the ability to handle different noise levels with a single model thanks to the parameter $\nu$ which controls the  strength of the denoising. This parameter has been estimated following the strategy recommended in \cite{zhang2021plug}. The parameter $\nu$ is uniformly sampled from a large noise level $\nu^{(1)}$ to a small one $\nu^{(N_{\mathrm{bi}})}$ according to a logarithmic scale, which results in a sequence of $\nu^{(1)} > \nu^{(2)} > \cdots > \nu^{(N_{\mathrm{bi}})}$. Following \cite{zhang2021plug}, $\nu^{(1)}$ is fixed to 49 while $\nu^{(N_{\mathrm{bi}})}$ is adjusted w.r.t.~the image noise level $\sigma$. For the sampling-based methods, to ensure the stationary of the kernel, the noise parameter is set to a fixed value beyond the burn-in period, i.e., $\forall t \in [N_{\mathrm{bi}}, N_{\mathrm{MC}}]$, $\nu^{(t)}=\nu^{(N_{\mathrm{bi}})}$. 

\subsection{Implementation details regarding RED-LwSGS}\label{app:impl_detail_LwSGS}
This appendix provides additional details regarding the implementation of the proposed LwSGS algorithm. The regularization parameter and the coupling parameters have been adjusted to reach the best performance.  For experiments on the FFHQ data set, the regularization parameter $\beta$ is set to $8.0\times10^{-2}, 1.25\times10^{-1}$ and $1.0$ for the deblurring, inpainting and super-resolution tasks, respectively, while it has been fixed to  $4.89\times10^{-3}$, $1.167\times10^{-1}$ and  $4.966\times10^{-2}$ for the experiments conducted on the ImageNet data set. The other parameters are fixed as $(N_{\mathrm{MC}}, N_{\mathrm{bi}}, \rho^2, \gamma)=(5000, 2000, 6\times10^{-8}, \tfrac{0.99}{2\beta + 1/\rho^2})$ for deblurring, $(N_{\mathrm{MC}}, N_{\mathrm{bi}}, \rho^2, \gamma)=(10000, 4500, 1.5, \tfrac{0.99}{2\beta + 1/\rho^2})$ for inpainting and $(N_{\mathrm{MC}}, N_{\mathrm{bi}}, \rho_1^2, \rho_2^2, \gamma) = (12500, 3500, 2\times10^{-1}, 1, \frac{0.8}{2\beta + 1/\rho_2^2})$ for super-resolution which requires a double splitting.

\subsection{Implementation details regarding the compared methods}\label{app:compared methods}
This appendix provides additional details regarding the implementation of the compared methods. First, RED-ADMM, RED-HQS, PnP-ADMM, PnP-ULA  and Diff-PIR are implemented using the same denoiser as the proposed method (see Appendix \ref{app:subsec:denoiser}). For the sampling-based methods, i.e., PnP-ULA, TV-MYULA, TV-SP, the total number of iterations have been set as for the proposed RED-LwSGS (see Appendix~\ref{app:impl_detail_LwSGS}). For optimization-based algorithms, i.e, RED-ADMM, RED-HQS and PnP-ADMM, the total number of iterations is set as follows: $150$ for deblurring and $350$ for inpainting and super-resolution. Finally all model and algorithmic parameters have been adjusted to reach the best performance. More precisely, for the experiments conducted on the FFHQ data:
\begin{itemize}
    \item RED-ADMM: the hyperparameters $(\alpha, \lambda, \beta)$ have been set to $(2, 2\times10^{-3}, 9\times10^{-4})$ for deblurring, $(2, 10^{-2}, 4\times10^{-2})$ for inpainting and $(2, 8\times10^{-3}, 10^{-6})$ for super-resolution.
    \item RED-HQS: similarly to RED-ADMM, the hyperparameters are $(\alpha, \lambda, \beta)=(2, 10^{-2}, 4\times10^{-3})$ for deblurring, $(\alpha, \lambda, \beta)=(2, 2\times10^{-2}, 1.8\times10^{-2})$ for inpainting and $(\alpha, \lambda, \beta)=(2, 8\times10^{-3}, 10^{-6})$ for super-resolution.
    \item PnP-ADMM: the parameter $\rho$ is set to $10^{-4}$ for deblurring, $10^{-3}$ for inpainting and $6.5\times10^{-2}$ for super-resolution.
    \item PnP-ULA: the parameters $(N_{\mathrm{bi}}, \beta)$ have been set to $(2500, 7.3\times10^{-4})$ for deblurring, $(5000, 10^{-4})$ for inpainting and $(6000, 2.75\times10^{-4})$ for super-resolution.
    \item TV-MYULA: the parameters $(N_{\mathrm{bi}}, \beta)$ have been set to $(2500, 10^{-1})$ for deblurring, $(5000, 8\times10^{-8})$ for inpainting and $(8500, 1)$ for super-resolution. 
    \item DiffPIR: the parameters $(\lambda, \zeta)$ have been set to $(2, 3\times10^{-1})$ for deblurring and $(1, 1)$ for inpainting and super-resolution.
    \item TV-SP: the parameters $(N_{\mathrm{bi}}, \rho, \beta)$ have been set to $(2000, 9\times10^{-4}, 1)$ for deblurring and $(4500, 1, 3.5\times10^{-3})$ for inpainting while for super-resolution, which requires a double splitting, the parameters $(N_{\mathrm{bi}}, \rho_1^2, \rho_2^2, \beta)$ have been set to $(3500, 10^{-5}, 10^{-5}, 3.76)$.    
\end{itemize}

\bibliographystyle{IEEEtran}
\bibliography{strings_all_ref,biblio}

\end{document}